\documentclass{article}

\PassOptionsToPackage{numbers,compress}{natbib}
\usepackage[preprint]{neurips_2026}

\usepackage[utf8]{inputenc} % allow utf-8 input
\usepackage[T1]{fontenc}    % use 8-bit T1 fonts
\usepackage{hyperref}       % hyperlinks
\usepackage{url}            % simple URL typesetting
\usepackage{booktabs}       % professional-quality tables
\usepackage{amsfonts}       % blackboard math symbols
\usepackage{nicefrac}       % compact symbols for 1/2, etc.
\usepackage{microtype}      % microtypography
\usepackage{xcolor}         % colors
\usepackage{amsmath}
\usepackage{amssymb}
\usepackage{mathtools}
\usepackage{amsthm}
\usepackage{graphicx}
\usepackage{subcaption}

\theoremstyle{plain}
\newtheorem{theorem}{Theorem}[section]
\newtheorem{proposition}[theorem]{Proposition}

\theoremstyle{definition}

\theoremstyle{remark}

\usepackage[capitalize,noabbrev]{cleveref}
\usepackage[textsize=tiny]{todonotes}
\usepackage[dvipsnames]{xcolor}

\usepackage{enumitem}

\usepackage{tikz-cd}
\usepackage[load-configurations=version-1]{siunitx}
\usepackage{tikz, pgfplots}
\pgfplotsset{compat=1.18}
\usetikzlibrary{arrows.meta, positioning, calc, fit, backgrounds, shadows, patterns}
\usepgfplotslibrary{groupplots, statistics}

\newcommand{\set}[1]{\mathcal{#1}}
\renewcommand*{\vec}[1]{\boldsymbol{#1}}
\newcommand{\ncae}{NcAE}
\newcommand{\cae}{cAE}
\newcommand{\contextcae}{Context-cAE}
\newcommand{\contextAE}{Context-AE}
\newcommand{\filmae}{FiLM+AE}
\newcommand{\naiveAE}{AE}
\newcommand{\softncae}{SoftNcAE}

\newcommand{\id}{\operatorname{id}}

\title{Neuromodulated Constrained Autoencoders\\ for Context-Dependent Manifold Learning }

\author{
  Jérôme Adriaens\\
  University of Liège\\
  Liège, Belgium\\
  \texttt{jadriaens@uliege.be}
  \And
  Gustave Bainier\\
  University of Liège\\
  Liège, Belgium\\
  \texttt{gustave.bainier@uliege.be}
  \And
  Guillaume Drion\\
  University of Liège\\
  Liège, Belgium\\
  \texttt{gdrion@uliege.be}
  \And
  Pierre Sacré\\
  University of Liège\\
  Liège, Belgium\\
  \texttt{p.sacre@uliege.be}
}

\begin{document}
\maketitle
% ============================================================
%  ABSTRACT
%  Covers: 
% (i) limitation of standard cAEs under context shift;
% (ii) NcAE as the only conditioning strategy that preserves idempotent
% projection properties across all contexts, with formal guarantees;
% (iii) empirical validation on pendulum and Lorenz96;
% ============================================================
\begin{abstract} 
Many physical systems exhibit a low-dimensional structure that varies with external parameters: link lengths in a robot, forcing constants in a fluid, or Reynolds numbers in a flow shift the underlying manifold while preserving its intrinsic dimension. Constrained AutoEncoders (cAEs) learn such manifolds through an \emph{idempotent} encoder-decoder projection, a property that unconstrained autoencoders cannot match and that is essential whenever the model is applied iteratively. However, the standard strategies for making a cAE context-dependent, namely concatenating the context to the input or affinely modulating hidden activations, break the encoder-decoder idempotency, sacrificing the projection guarantee precisely in the setting where it would be most valuable. To restore this guarantee under context variation, we developed the Neuromodulated Constrained Autoencoder (NcAE), which modulates the activation slope and bias of a cAE through a context-driven hyper-network. This paper presents the NcAE, its theoretical foundation, and its empirical validation. We prove that for every context, including contexts unseen at training time, the reconstruction map remains an idempotent projection, the topology of the learned manifold is invariant, and context perturbations induce smooth changes in the manifold. We evaluated our approach on a 16-DoF pendulum with context-dependent coupling and the Lorenz96 system across a bifurcation. The NcAE matched or exceeded the best of six baselines on reconstruction, idempotency, and latent-geometry metrics, while being the only architecture that preserves geometric consistency by construction. The NcAE thereby provides a stable, geometry-preserving coordinate system across families of physical regimes.
\end{abstract}

\section{Introduction}
\label{sec:intro}
% context
Dimensionality reduction is foundational to representation learning for high-dimensional data. Among nonlinear methods, AutoEncoders (\naiveAE s)~\citep{wangAutoencoderBasedDimensionality2016} offer the expressive power to capture complex geometries that linear approaches such as PCA cannot. 
Building on this, Constrained AutoEncoders (\cae s)~\citep{otto_learning_2023} enforce biorthogonality between encoder and decoder weights together with mutually inverse activations, making the reconstruction map an \emph{idempotent projection} onto a smooth embedded submanifold of the data space.
Idempotency matters most when the autoencoder is applied iteratively; for instance, when integrating a dynamical system in latent space or maintaining physical consistency over long trajectories.
In this regime, an idempotent map yields one-step stability: once a state lies on the learned manifold, repeated applications leave it there, eliminating the projection drift that unconstrained \naiveAE s accumulate at every step.

% need
Many physical systems do not live on a single, static manifold. Their geometry depends on external parameters (link lengths in a robot, forcing constants in a fluid, Reynolds numbers in a flow), and these parameters can vary across regimes that share an underlying structure. 
The standard remedies---input concatenation, or affine activation modulation as in FiLM \citep{perez_film_2017} and related schemes \citep{karrasStyleBasedGeneratorArchitecture2019,almachotHFMHybridFeature2022}---both break the encoder–decoder inverse relationship that makes a cAE a projection. The geometric guarantee is lost precisely where it would be most valuable: across families of physical regimes whose long-horizon consistency depends on the projection structure.
The geometric guarantee that motivated the \cae{} is therefore lost precisely in the setting where it would be the most valuable, across families of physical regimes whose long-horizon consistency was meant to be ensured by the projection structure.
A more detailed related work is provided in \cref{sec:related-work}.

% task
We characterized the \cae{} parameters that admit context-dependent modulation without additional architectural machinery: the activation slope and the layer bias can be modulated freely while preserving the encoder–decoder idempotency, whereas modulating the biorthogonal weights would require an explicit parameterization of the biorthogonal manifold. 
Modulating the activation slope and bias through a context-driven hyper-network yielded the Neuromodulated Constrained AutoEncoder (\ncae), for which we proved three geometric guarantees that hold at every context, including those unseen at training time: idempotent projection, diffeomorphic invariance of the learned manifold across contexts, and Lipschitz stability of the learned manifold under context perturbation. 
We evaluated the \ncae{} on a 16-DoF pendulum~\citep{friedl_riemannian_2025} with explicit context-dependent coupling and on the Lorenz96 system~\citep{75462} across a bifurcation in its forcing parameter, against six baselines spanning context-free, input-concatenation, FiLM-style, and soft-constraint variants. 
The \ncae{} matched or exceeded the best baseline on every metric while preserving geometric consistency by construction.

% This paper is organized as follows.
% \cref{sec:background} recalls the \cae{} framework and its geometric guarantees. %outlines the mathematical preliminaries for \cae s.
% \cref{sec:ncae} introduces the \ncae{} and establishes its theoretical properties.
% \cref{sec:exp} reports the experimental evaluation against six baselines, including reconstruction, idempotency, latent geometry, and robustness to noise and out-of-distribution context.
% \cref{sec:discussion} discusses implications and limitations, followed by concluding remarks in \cref{sec:conclusion}.

% ============================================================
%  SECTION 2 — BACKGROUND AND PRELIMINARIES
%  Preserved from current draft, with one key addition:
%  - New paragraph in §2.2 motivating WHY idempotency matters
%    for dynamical systems (answers Reviewer EsPU):
%    idempotency ≠ perfect reconstruction; idempotency gives
%    one-step stability preventing state drift under repeated
%    application — the correct geometric substitute for a
%    bottleneck architecture.
% ============================================================
% ============================================================
%  Section 2 — Background and Preliminaries
% ============================================================

\section{Background and preliminaries}
\label{sec:background}

% This section defines the mathematical framework for standard AEs and \cae s, and discusses the geometric constraints required to transform them into formal projection operators.

\subsection{Autoencoders as nonlinear mappings}
\label{sec:classicae}

We assume the data lies on or near a low-dimensional manifold $\mathcal{M} \subset \mathbb{R}^n$ of dimension $m \ll n$.
An AE seeks an approximation $\hat{\mathcal{M}}$ of $\mathcal{M}$ together with explicit encoder and decoder maps relating the ambient space $\mathbb{R}^n$ to a latent space $\mathbb{R}^m$.
The encoder $\rho:\mathbb{R}^n \to \mathbb{R}^m$ maps each input to a latent coordinate $\vec{z} = \rho(\vec{x})$; the decoder $\varphi:\mathbb{R}^m \to \mathbb{R}^n$ maps it back to the ambient space; the composition $P \coloneqq \varphi \circ \rho :\mathbb{R}^n \to \mathbb{R}^n$ is the reconstruction map.
The encoder and the decoder are trained jointly to minimize the reconstruction loss
\begin{equation*} \label{eq:ae-loss}
    L(\rho,\varphi) = E_{\vec{x} \sim \mathcal{D}} \left[ \| \vec{x} - \varphi(\rho(\vec{x})) \|^2 \right],
\end{equation*}
where $\mathcal{D}$ is the data distribution.
%, and which penalises the discrepancy between each input and its reconstruction.
The image $\hat{\mathcal{M}} \coloneqq P(\mathbb{R}^n)$ is the \emph{learned manifold}: ideally, the data lies on or near $\hat{\mathcal{M}}$, and any state can be encoded, manipulated in latent space, and decoded back to the ambient space.

\subsection{From autoencoders to projections} \label{sec:cae}

Standard autoencoders impose no geometric structure on $P$ beyond the reconstruction loss. In particular, $P$ is not generally idempotent: $P \circ P \neq P$. This matters whenever the autoencoder is applied iteratively, as in latent-space integration of dynamical systems, where small per-step deviations from the manifold accumulate over trajectories. In a bottleneck setting ($m \ll n$), exact identity $P(\vec{x}) = \vec{x}$ is not achievable for all $\vec{x}$, since the encoder irreversibly discards information. The natural geometric goal is therefore not perfect reconstruction but \emph{idempotency}: a fixed-point condition $P \circ P = P$ guaranteeing that once a state lies on $\hat{\mathcal{M}}$, repeated application of the projection leaves it unchanged~\citep{shocherIdempotentGenerativeNetwork2023}.
The constrained autoencoder (\cae) of \citet{otto_learning_2023} achieves idempotency by enforcing the encoder as a left inverse of the decoder 
\begin{equation} \label{eq:leftinv}
    \rho \circ \varphi = \id_{\mathbb{R}^m}.
\end{equation} 
Under \eqref{eq:leftinv}, the reconstruction map satisfies
\begin{equation*}
        P \circ P = \varphi \circ (\rho \circ \varphi) \circ \rho = \varphi \circ \mathrm{id}_{\mathbb{R}^m} \circ \rho = P,
\end{equation*}
so $P$ is an idempotent projection and $\hat{\mathcal{M}} = \operatorname{Range}(P)$ is a smooth embedded submanifold of $\mathbb{R}^n$ \citep{michor_topics_2008}.

The \cae{} realizes \eqref{eq:leftinv} at the network level by composing $L$ paired layers,
\begin{equation*}
    \rho = \rho^{(1)}\circ\dots\circ\rho^{(L)}, \qquad  \varphi = \varphi^{(L)}\circ\dots\circ\varphi^{(1)},
\end{equation*}
indexed so that $\rho^{(l)}$ and $\varphi^{l)}$ form the layer pair at depth $l$. Each pair takes the form 
\begin{align*}
    \rho^{(l)}(\vec{x}^{(l)}) &= \sigma_-\! \left( \vec{\Psi}_l^T (\vec{x}^{(l)} - \beta) ;\, \alpha \right), \\
    \varphi^{(l)}(\vec{z}^{(l-1)}) &= \vec{\Phi}_l \sigma_+\! \left(\vec{z}^{(l-1)} ;\, \alpha \right) + \beta.
\end{align*}
Three ingredients together yield $\rho^{(l)} \circ \varphi^{(l)} = \mathrm{id}$ at every layer:
\begin{enumerate}[label=(A\arabic*)]
    \item \emph{Biorthogonal weights}. The encoder and decoder weight matrices satisfy $\vec{\Psi}_l^T \vec{\Phi}_l  = I$, so that the linear part of the layer pair composes to the identity on $\mathbb{R}^{n_l}$. \label{ass:biorth}
    
    \item \emph{Inverse activation pair}. The activations $\sigma_-(\,\cdot\,;\alpha)$ and $\sigma_+(\,\cdot\,;\alpha)$ are smooth and mutually inverse for every $\alpha$ in an admissible range (see \cref{sec:act-func}). Geometrically, $\alpha$ controls the nonlinearity: $\alpha$ near $0$ yields nearly linear behavior, while $\alpha$ approaching $\pi/4$ yields strong nonlinear shaping.
    \label{ass:inverse}
    
    \item \emph{Shared bias}. Both layers in the pair use the same bias $\beta$, so the affine offsets cancel under composition.
    \label{ass:bias}
    
\end{enumerate}

By composition across layers, $\rho \circ \varphi = \mathrm{id}_{\mathbb{R}^m}$, ensuring that $P = \varphi \circ \rho $ is an idempotent projection with $\hat{\mathcal{M}}$ a smooth embedded $m$-submanifold of $\mathbb{R}^n$.

% ============================================================
%  SECTION 3 — NEUROMODULATED CONSTRAINED AUTOENCODER (NcAE)
%  §3.1–3.3: preserved from current draft
%  §3.4 NEW: Theoretical guarantees surfaced into main paper
%    (stated as propositions, proofs in appendix)
%    - Prop 1: Idempotency for all c
%    - Prop 2: Topological invariance across contexts
%    - Prop 3: Hausdorff stability
% ============================================================
% ============================================================
%  Section 3 — Neuromodulated Constrained Autoencoder (NcAE)
% ============================================================

\section{Neuromodulated constrained autoencoder}
\label{sec:ncae}

% This section describes the \ncae{} by first defining the role of neuromodulation in context-dependent adaptation, then explaining the shared embedding used for scalability, and finally detailing the layer-specific parameters that govern the manifold transformations.
% This architecture enables the network to learn a family of manifolds parameterized by an external context vector $\vec{c}$ while strictly maintaining the idempotency property established in \cref{sec:background}.

\subsection{Which parameters can be modulated?}
\label{ssec:neuromod}

A \cae{} acquires its projection structure from three architectural ingredients: biorthogonal weights~\ref{ass:biorth}, an inverse activation pair~\ref{ass:inverse}, and a shared layer bias~\ref{ass:bias}.
Together, these guarantee that $\rho \circ \varphi = \id_{\mathbb{R}^m}$ and hence that $P = \varphi \circ \rho$ is an idempotent projection. 
To make the \cae{} context-dependent, one would like to replace some of these parameters with functions of an external context vector $\vec{c}$. The question is which replacements preserve the inverse relation $\rho_{\vec{c}} \circ \varphi_{\vec{c}} = \id_{\mathbb{R}^m}$ for every~$\vec{c}$.

In the nervous system, \emph{neuromodulation} is the process by which diffuse signals adjust the response properties of neurons (gain, threshold, time constant) without altering the underlying connectivity~\citep{bargmann_connectome_2013}.
This separation between fixed structure and adjustable response has informed several recent ML approaches to context-dependent computation \citep{vecoven_introducing_2020,meiInformingDeepNeural2022}.
For a \cae{}, this suggests modulating \emph{response parameters} (the activation slope $\alpha$ and the bias $\beta$) while leaving the \emph{structural parameters} (the biorthogonal weights) fixed. 
Under this choice, the encoder-decoder idempotency holds for every $\vec{c}$, and the resulting reconstruction map $P_{\vec{c}}$ remains an idempotent projection (\cref{prop:idempotency}). Modulating the biorthogonal weights $\vec{\Psi}_l, \vec{\Phi}_l$ would also preserve idempotency, but only if the hyper-network output remained on the biorthogonal manifold for every $\vec{c}$, requiring an explicit parameterization of that manifold. We do not pursue this extension here.
Concatenating $\vec{c}$ to the input makes the projection a function of $(\vec{x},\vec{c})$ jointly, breaking idempotency in the original state space; affine modulation of hidden activations (FiLM \citep{perez_film_2017}, AdaIN \citep{karrasStyleBasedGeneratorArchitecture2019}) destroys the inverse pair $(\sigma_-,\sigma_+)$. Modulating $(\alpha,\beta)$ bypasses both failure modes while keeping the hyper-network unconstrained.

\begin{figure}
    \centering
    \begin{tikzpicture}[scale=1.0]
    \def\nmdcolor{orange}
    \def\commonnmdcolor{purple}
    \def\enccolor{teal}
    \def\deccolor{violet}
    \tikzset{
        my shadow/.style={#1, copy shadow={opacity=0.3, #1}},
        common_nmd/.style={label={[font=\small, shift=(nmd_box.north west), \commonnmdcolor] above right:Common neuromodulation}},
        spec_nmd/.style={label={[font=\small, shift=(nmd_box_local.north east), \nmdcolor, align=left]above left:Layer specific\\ neuromodulation}}
    }
    % Main nodes
    \node (context) at (-2,0) {$\vec{c}$};
    \node[draw, rectangle, fill=gray!10, right=of context, align=center, minimum width=1cm] (nmd_fc) {FC\\ NN};
    \node[right=of nmd_fc] (nmd_signal) {$\vec{s}$};
    \node[below right=.8cm and 0.0cm of nmd_signal] (layer_nmd) {$\vec{\alpha}^{(l)} = g(\vec{W}^T_{l,\alpha} \vec{s})$};
    % Add a second node similar to layer_nmd, touching it on the right side
    \node[right=8pt of layer_nmd, anchor=west] (layer_nmd2) {$\vec{\beta}^{(l)} = \vec{W}^T_{l,\beta} \vec{s} + \vec{b}_l$};
    % Optionally, draw a connection or highlight the touching sides
    % \draw[thick] (layer_nmd.east) -- (layer_nmd2.west);
    \begin{scope}[on background layer]
        \node[fit=(context)(nmd_fc)(nmd_signal), draw, dashed, thick, inner sep=4pt,\commonnmdcolor,fill = \commonnmdcolor!10,{common_nmd}] (nmd_box) { };
        \node[copy shadow={opacity=0.3}, fit=(layer_nmd)(layer_nmd2), draw, dashed, thick, inner sep=4pt,\nmdcolor,fill=\nmdcolor!10,{spec_nmd}] (nmd_box_local) {};
    \end{scope}
    % Draw a dashed line from top to bottom of nmd_box_local at the midpoint of layer_nmd and layer_nmd2
    \draw[dashed, thick, \nmdcolor]
        ({$(layer_nmd.east)!0.5!(layer_nmd2.west)$} |- nmd_box_local.south) -- 
        ({$(layer_nmd.east)!0.5!(layer_nmd2.west)$} |- nmd_box_local.north);

    \node[below=of {$(layer_nmd.east)!0.5!(layer_nmd2.west)$} |- nmd_box_local.south ] (blank_nmd) {};

    \draw[->, thick] (context) -- (nmd_fc);
    \draw[->, thick] (nmd_fc) -- (nmd_signal);
    \draw[thick, ->] (nmd_signal) -| (nmd_box_local);

    \node[below =1cm of blank_nmd] (blank) {$\vec{x}^{(0)}=\vec{z}^{(0)}$};
    \node[left=of blank] (ll_dots) {$\cdots$};
    \node[right=of blank] (rr_dots) {$\cdots$};
    \node[right=of rr_dots] (r_dots) {$\cdots$};
    \node[right=of r_dots] (output) {$\vec{z}^{(L)} = \vec{z}$};
    \node[left=of ll_dots] (l_dots) {$\cdots$};

    \draw[->, thick] (l_dots) -- node[above] (rho_l) {$\rho^{(l)}$} (ll_dots);
    \draw[->, thick] node[left=of l_dots] (input) {$\vec{x} = \vec{x}^{(L)}$} (input) -- node[above] (rho_L) {$\rho^{(L)}$}(l_dots);
    \draw[->, thick] (ll_dots) -- node[above] (rho_0) {$\rho^{(1)}$} (blank);

    \draw[->, thick] (rr_dots) -- node[above] (varphi_l) {$\varphi^{(l)}$} (r_dots) ;
    \draw[->, thick] (r_dots) -- node[above] (varphi_L) {$\varphi^{(L)}$} (output);
    \draw[->, thick] (blank) -- node[above] (varphi_0) {$\varphi^{(1)}$} (rr_dots);

    \node[below=0.3cm of {$(l_dots)!0.5!(ll_dots)$}, \enccolor] {Encoder}; 
    \node[below=0.3cm of {$(rr_dots)!0.5!(r_dots)$}, \deccolor] {Decoder};

    \node[dashed, above=.7cm of rho_L, minimum height=1.cm, draw=\enccolor, fill=\enccolor!10, my shadow] (enc_nmd) 
    {$\sigma_-\left(\vec{\Psi}^T_l(\vec{x}^{(l)} - {\color{\nmdcolor}\vec{b}^{(l)}}); {\color{\nmdcolor}\vec{\alpha}^{(l)}}\right)$};
    \node[dashed, above=.7cm of varphi_L, minimum height=1.cm, draw=\deccolor, fill=\deccolor!10, my shadow] (dec_nmd) 
    {$\vec{\Phi}_l\sigma_+\left((\vec{z}^{(l-1)}); {\color{\nmdcolor}\vec{\alpha}^{(l)}}\right) + {\color{\nmdcolor}\vec{\beta}^{(l)}}$};
    % Draw an arrow from layer_nmd to a point between enc_nmd.north and enc_nmd.north east
    \draw[->, thick, \nmdcolor] (nmd_box_local) -| ($(enc_nmd.north)!0.0!(enc_nmd.north east)$);
    \draw[->, thick, \nmdcolor] (nmd_box_local) -| ($(dec_nmd.north)!0.0!(dec_nmd.north east)$);
    
    \draw[\enccolor] (rho_l.north) -- (enc_nmd.south west);
    \draw[\enccolor] (rho_l.north) -- (enc_nmd.south east);
    \draw[\deccolor] (varphi_l.north) -- (dec_nmd.south west);
    \draw[\deccolor] (varphi_l.north) -- (dec_nmd.south east);
    \begin{scope}[on background layer]
        \draw[\enccolor, opacity=0.1] (rho_L.north) -- ($(enc_nmd.south west) + (.5ex, .5ex)$);
        \draw[\enccolor, opacity=0.1] (rho_L.north) -- ($(enc_nmd.south east) + (.5ex, .5ex)$);
        \draw[\deccolor, opacity=0.1] (varphi_L.north) -- ($(dec_nmd.south west) + (.5ex, .5ex)$);
        \draw[\deccolor, opacity=0.1] (varphi_L.north) -- ($(dec_nmd.south east) + (.5ex, .5ex)$);
        \draw[\enccolor, opacity=0.1] (rho_0.north) -- ($(enc_nmd.south west) + (.5ex, .5ex)$);
        \draw[\enccolor, opacity=0.1] (rho_0.north) -- ($(enc_nmd.south east) + (.5ex, .5ex)$);
        \draw[\deccolor, opacity=0.1] (varphi_0.north) -- ($(dec_nmd.south west) + (.5ex, .5ex)$);
        \draw[\deccolor, opacity=0.1] (varphi_0.north) -- ($(dec_nmd.south east) + (.5ex, .5ex)$);
        % Draw a dashed line through the equal sign between x^{(0)} and z^{(0)}, connecting to encoder and decoder
        \draw[dashed, opacity=0.5] 
            ($(blank.north) + (0.0,1.0)$) -- 
            ($(blank.south) + (0.0,-0.7)$);
    \end{scope}
\end{tikzpicture}
    \caption{
    Neuromodulation dynamically reconfigures the autoencoder manifold by parameterizing activation functions and biases based on external context.
    This schematic illustrates the integration of the context vector $\vec{c}$, which a fully connected MLP processes to generate the modulation signal $\vec{s}$. This signal is then multiplied by a layer-specific weight matrix $\vec{W}_{l, \alpha}$ and passed through the saturating function $g$ to compute the activation parameters $\vec{\alpha}^{(l)}$. Simultaneously, $\vec{s}$ passes through a linear layer defined by parameters $\vec{W}_{l, \beta}$ and $\vec{b}_l$ to determine the layer-wise bias $\vec{\beta}^{(l)}$.
    }
    \label{fig:neuromodulation}
\end{figure}
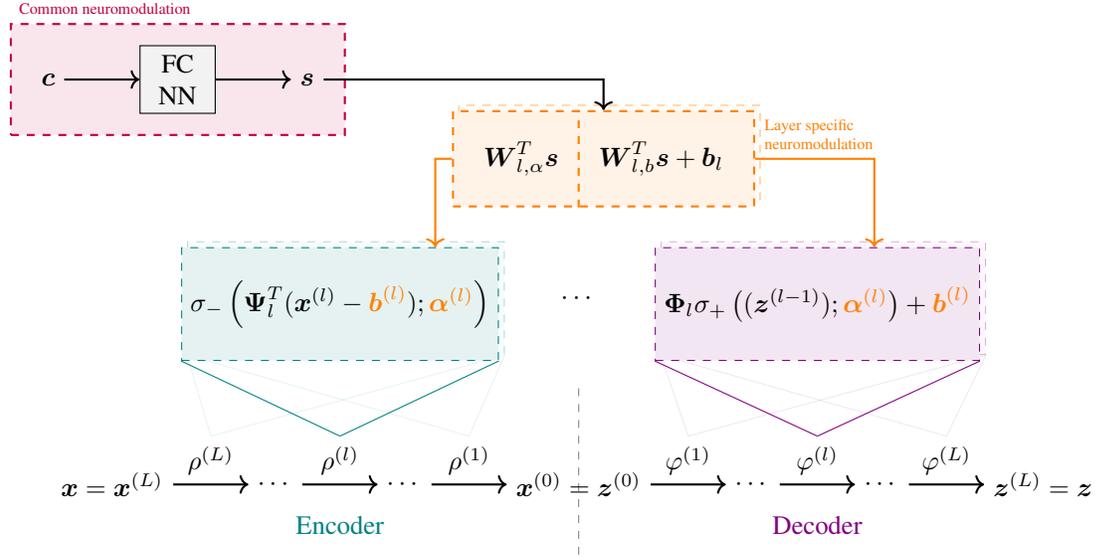

\subsection{From context to layer parameters}

To realize the modulation of $(\alpha,\beta)$ with a small number of parameters, we generate $(\vec{\alpha}^{(l)},\vec{\beta}^{(l)})$ from $\vec{c}$ through a two-stage hyper network. A scheme of the full architecture is depicted in \cref{fig:neuromodulation}.
The first stage is shared across all layer pairs: a fully connected network maps $\vec{c}$ to a low-dimensional embedding $\vec{s} = f_\text{nmd}(\vec{c};\theta) \in \mathbb{R}^d$, where $d$ is small relative to the total dimension of $(\vec{\alpha}^{(l)},\vec{\beta}^{(l)})$ over all layers. 
This shared embedding serves as a global neuromodulatory signal that captures the context-dependent information once and distributes it across layers.
Beyond reducing parameter count, this bottleneck regularizes the modulation: it prevents per-layer overfitting to specific context values and encourages a coherent context representation across the network.
The second stage maps $\vec{s}$ to layer-specific parameters.
For each layer pair $l$, we set
\begin{align*}
    \vec{\alpha}^{(l)} &= g(\vec{W}^T_{l,\alpha} \vec{s}),\\
    \vec{\beta}^{(l)} &= \vec{W}^T_{l,\beta} \vec{s} + \vec{b}_l,
\end{align*}
where $\vec{W}_{l,\alpha}$, $\vec{W}_{l,\beta}$, and $\vec{b}_l$ are trainable layer-specific weights and biases.
The saturating sigmoid $g$ in the $\vec{\alpha}^{(l)}$ expression keeps the slope inside the admissible interval $[\alpha_\text{min},\alpha_\text{max}]$ for every $\vec{c}$ such that
\begin{equation*}
    \vec{\alpha}^{(l)} = (\alpha_{\max}-\alpha_{\min})\cdot\textsc{sigmoid}(\vec{W}_{l, \alpha}^T \vec{s}) + \alpha_{\min},
\end{equation*}
which ensures the inverse activation property holds globally. 
The $\vec{\beta}^{(l)}$ expression is unconstrained. 

% The two modulation channels play geometrically distinct roles. 
% Bias modulation through $\vec{\beta}^{(l)}$ translates the learned manifold in the ambient space, shifting the projection center as the context varies; this captures rigid relocations of $\hat{\mathcal{M}}_{\vec{c}}$.
% Slope modulation through $\vec{\alpha}^{(l)}$ deforms the manifold by stretching or compressing its internal coordinate structure; this captures the nonlinear geometric changes that occur across different dynamical regimes. 
% Although each modulation type alone can tune the manifold individually, we show empirically in \cref{ssec:ablation} that both channels are necessary.

\subsection{Theoretical guarantees}
\label{ssec:theory}

The neuromodulation mechanism preserves the \cae{} geometric structure for \emph{all} values of $c$, not just those seen during training. We state three results here; full proofs are in \cref{sec:appendix-theory}.

The propositions rely on the assumptions \ref{ass:biorth}–\ref{ass:bias} inherited from the \cae{} construction (\cref{sec:cae}) together with one assumption specific to the modulated setting:
\begin{enumerate}[label=(A\arabic*)]
    \setcounter{enumi}{3}
    \item \emph{Smoothness of the modulation}. $\sigma_\pm$ is smooth in $z$ and $\alpha$, and the hyper-network $\vec{c} \mapsto ({\vec\alpha}^{(l)}(\vec{c}), \vec{\beta}^{(l)}(\vec{c}))$ is smooth. \label{ass:smooth}
\end{enumerate}

Assumption \ref{ass:smooth} follows from the analytic form of $\sigma_\pm$  and from the smooth activations of the hyper-network. 
Throughout, we write $\hat{\mathcal{M}}_{\vec{c}} \coloneqq \operatorname{Range}(P_{\vec{c}})$ for the manifold learned at context~$\vec{c}$.

\begin{proposition}[Idempotency]
\label[proposition]{prop:idempotency}
Under \ref{ass:biorth}--\ref{ass:bias}, $P_{\vec{c}} = \varphi_{\vec{c}} \circ \rho_{\vec{c}}$ is an idempotent projection for every
context $\vec{c}$, i.e., $P_{\vec{c}} \circ P_{\vec{c}} = P_{\vec{c}}$.
% Combined with \ref{ass:smooth}, $\hat{\mathcal{M}}_{\vec{c}}$ is a smooth embedded $m$-submanifold of $\mathbb{R}^n$. 
\end{proposition}
Idempotency holds at every context, including contexts unseen during training.
Once a state lies on $\hat{\mathcal{M}}_{\vec{c}}$, repeated application of $P_{\vec{c}}$ leaves it there; the operator does not accumulate drift over iterates
This one-step stability underlies the long-horizon idempotency experiments of \cref{ssec:idempotency}.

\begin{proposition}[Diffeomorphic invariance of the learned manifold]
\label[proposition]{prop:topo}
Under \ref{ass:biorth}--\ref{ass:smooth}, for any two contexts $\vec{c}_1$ and $\vec{c}_2$, the map $\varphi_{\vec{c}_2} \circ \rho_{\vec{c}_1}:\mathcal{\hat M}_{\vec{c}_1} \to \mathcal{\hat M}_{\vec{c}_2}$ defines a diffeomorphism from $\mathcal{\hat M}_{\vec{c}_1}$ to $\mathcal{\hat M}_{\vec{c}_2}$, of inverse $(\varphi_{\vec{c}_2} \circ \rho_{\vec{c}_1})^{-1} = \varphi_{\vec{c}_1} \circ \rho_{\vec{c}_2} : \mathcal{\hat M}_{\vec{c}_2} \to \mathcal{\hat M}_{\vec{c}_1}$. In particular, $\hat{\set{M}}_{\vec{c}_1}$ and $\hat{\set{M}}_{\vec{c}_2}$ are topologically equivalent for all context pairs. 
\end{proposition}
This proposition states that the learned manifold has the same intrinsic structure across all contexts: its dimension, topology, and smooth structure are preserved as $c$ varies. Whatever context the \ncae{} is presented with (including a context unseen during training), the recovered $\hat{\mathcal{M}}_{\vec{c}}$ is the same manifold up to smooth deformation, never a manifold of a different shape or topology. This is the correct inductive bias for physical systems whose state space evolves with parameters but whose intrinsic dimensionality is fixed.

\begin{proposition}[Lipschitz stability under context perturbation]
\label[proposition]{prop:hausdorff}
Fix a reference context $\vec{c}$, a compact set $K \subset \mathbb{R}^n$, and a maximal perturbation size $p > 0$. Let $\mathcal{P} \triangleq \{\delta \vec{c} : \lVert \delta \vec{c} \rVert \leq p\}$ denote the set of context perturbations.
Under \ref{ass:biorth}--\ref{ass:smooth}, there exists a Lipschitz constant $L = L(K, \vec{c}, p) > 0$ such that for all $\delta\vec{c}\in \mathcal{P}$:
\begin{equation*}
    d_H\!\left(\hat{\set{M}}_{\vec{c}} \cap K,\; \hat{\set{M}}_{\vec{c}+\delta\vec{c}} \cap K\right) \leq L\,\|\delta\vec{c}\|,
\end{equation*}
where $d_H$ denotes the Hausdorff distance in $\mathbb{R}^n$.
\end{proposition}
Whereas \cref{prop:topo}  establishes that the family $\hat{\mathcal{M}}_{\vec{c}}$ consists of diffeomorphic manifolds being topologically invariant, \cref{prop:hausdorff} establishes that this family is continuous: small context perturbations (measurement noise on $\vec{c}$, or out-of-distribution context values near the training boundary) induce only proportionally small changes in the \emph{position} of the manifold in ambient space. 
The bound is qualitative: the constant~$L$ is not minimized during training, and tightening it would require regularizing $L$ in the objective. 
This property directly supports the robustness analysis in \cref{ssec:robustness}. 
 
% ============================================================
%  SECTION 4 — SINGLE NcAE EXPERIMENTS
%  Validates the single NcAE architecture before introducing
%  the Paired NcAE extension.
%
%  §4.1 Experimental setup (systems, baselines, metrics)
%  §4.2 Reconstruction and latent analysis
%  §4.3 Ablation: modulation type
%  §4.4 Robustness analysis (OOD + noise)
% ============================================================
% ============================================================
%  Section 4 — Single NcAE Experiments
% ============================================================

\section{Experiments}
\label{sec:exp}

% This section validates the \ncae{} against a comprehensive set of baselines across two dynamical systems, with analyses of reconstruction fidelity, geometric consistency, modulation design, and robustness.

\subsection{Experimental setup}
\label{ssec:setup}
We evaluate the  \ncae{} on two dynamical systems and against six baselines. The two systems probe complementary regimes of context-dependent geometry; the six baselines isolate, one at a time, each design choice the NcAE makes differently.

\paragraph{Systems.}
The 16-DoF pendulum~\citep{friedl_riemannian_2025} provides a controlled setting where context-dependent coupling is imposed by design: the lengths $l_1,\ldots,l_4$ of the first four links are sampled from $[0.35, 0.65]\si{\meter}$ and govern the coupling of the remaining 12 DoFs through nonlinear functions.
The Lorenz96 system~\citep{75462} of dimension $n=36$ features an intrinsic bifurcation at forcing constant $F \approx 3.163$, within the range $F \in [3.133, 3.193]$, where the attractor topology fundamentally changes.
The pendulum tests adaptation to externally imposed manifold deformations; Lorenz96 tests adaptation to manifold deformations that emerge from the dynamics themselves.
More details are provided in \cref{sec:exp-details}.

\paragraph{Baselines.}
Each baseline isolates a single design dimension. 
Without context: cAE (constrained) and AE (unconstrained) are the lower bounds for what each architectural family can achieve without conditioning. 
With input concatenation: \contextcae{} and \contextAE{} apply the standard remedy for context dependence to a constrained and an unconstrained backbone, respectively. 
With activation conditioning: \filmae{} equips an unconstrained backbone with FiLM-style affine modulation of hidden activations~\citep{perez_film_2017}. 
With soft constraints: \softncae{} is identical to \ncae{} but enforces biorthogonality through a Euclidean penalty in the loss rather than through Riemannian optimization~\citep{kochurovGeooptRiemannianOptimization2020}, testing whether the hard geometric constraint is necessary.

\paragraph{Metrics.}
We report four metrics, each tied to a specific question. 
Reconstruction RMSE on position and velocity measures fidelity in the ambient space. 
Idempotency error $\|P(P(\vec{x})) - P(\vec{x})\|$ over $k$ successive projection steps tests \cref{prop:idempotency}. 
Condition number (CN) of the latent covariance matrix and coefficient of variation (CV) of the latent velocity norm probe latent geometry: CN near 1 indicates isotropic axis scaling, low CV indicates homogeneous projected dynamics. 
The OOD context sweep and the noise sweep test \cref{prop:hausdorff}.

\paragraph{Evaluation protocol.}
Every architecture is trained independently with 10 different random seeds to account for model variability due to random initialization and stochastic optimization.
Reported means and standard deviations are computed across these 10 runs.
Hyperparameters and data-generation details are in \cref{sec:appendix-lorenz,sec:appendix-pendulum}.

\subsection{Reconstruction performance}
\label{ssec:reconstruction}
NcAE delivers the best overall reconstruction on both systems, with the largest margin on velocity in the bifurcation regime, while context-aware unconstrained baselines collapse precisely where the manifold deforms most. \cref{tab:reconstruction-results} reports reconstruction RMSE for all architectures on the pendulum and Lorenz96 systems.
\begin{table}
    \caption{%
    \ncae{} achieves the best overall reconstruction across both systems and both metrics. \filmae{} matches \ncae{}  on pendulum position but with five times the variance; unconstrained autoencoders collapse on velocity in the bifurcating regime. 
    RMSE reported as mean $\pm$ std over 10 seeds.
    }
    \label{tab:reconstruction-results}
    \centering\medskip
    \begin{tabular}{lcccc}
        \toprule
        & \multicolumn{2}{c}{\textbf{Pendulum}} & \multicolumn{2}{c}{\textbf{Lorenz96}} \\
        \cmidrule(lr){2-3} \cmidrule(lr){4-5}
        \textbf{Method} & \textbf{Pos. RMSE} $\downarrow$ & \textbf{Vel. RMSE} $\downarrow$ & \textbf{Pos. RMSE} $\downarrow$ & \textbf{Vel. RMSE} $\downarrow$ \\
        \midrule
        \cae            & $0.049{\pm 0.001}$ & $0.220{\pm 0.002}$ & $0.468{\pm 0.25}$ & $0.653{\pm 0.1}$ \\
        \contextcae     & $0.054{\pm 0.001}$ & $0.226{\pm 0.001}$ & $0.319{\pm 0.254}$ & $0.644{\pm 0.204}$ \\
        AE              & $0.050{\pm 0.001}$ & $2.252{\pm 3.506}$ & $0.177{\pm 0.090}$ & $18.634{\pm 8.044}$ \\
        \contextAE       & $0.051{\pm 0.002}$ & $0.883{\pm 0.803}$ & $0.237{\pm 0.046}$ & $11.137{\pm 2.797}$ \\
        \filmae{}         & $\mathbf{0.008{\pm 0.01}}$ & $0.348{\pm 0.754}$ & $0.323{\pm 0.073}$ & $2.825{\pm 2.676}$ \\
        SoftNcAE        & $0.026{\pm 0.005}$ & $0.114{\pm 0.019}$ & $0.748{\pm 0.227}$ & $2.238{\pm 0.543}$ \\
        \midrule
        \ncae{} (ours)  & $0.012{\pm 0.002}$ & $\mathbf{0.059{\pm 0.003}}$ & $\mathbf{0.079{\pm 0.05}}$ & $\mathbf{0.329{\pm 0.186}}$ \\
        \bottomrule
    \end{tabular}
\end{table}

%Across both systems, the \ncae{} achieves the best overall reconstruction performance, with the advantage most pronounced in velocity and in the Lorenz96 bifurcation regime.
On the pendulum, \ncae{} reduces position RMSE by approximately 75\% relative to the \cae{} and \contextcae{} and achieves the lowest velocity RMSE.
\filmae{} produces a marginally lower mean position RMSE ($0.008$), but its variance is five times larger ($\pm 0.010$ versus $\pm 0.002$ for \ncae), indicating more discrepancies in performances across training runs that are absent from the geometrically constrained model.
\contextcae{} fails to improve over the context-free \cae{} on either metric, confirming that input augmentation alone cannot resolve shifting a physical manifold.

On Lorenz96, \ncae{} achieves the lowest position and velocity RMSE, with an order-of-magnitude advantage in velocity over every unconstrained baseline.
SoftNcAE shows notably higher variance on both systems, a first indication that hard biorthogonality enforced through Riemannian optimization is essential for stable training rather than merely convenient.

\subsection{Idempotency analysis}
\label{ssec:idempotency}

Idempotency holds by construction in \ncae{} and breaks in every context-aware baseline, validating \cref{prop:idempotency}. 
\cref{fig:idempotency-bars} reports idempotency error $\|P(P(\vec{x})) - P(\vec{x})\|$ over $k$ successive projection steps.
\cae{} and \ncae{} maintain near-zero error (${\sim}10^{-5}$) on both systems.
\contextcae{} breaks idempotency despite its constrained backbone: concatenating the context to the input modifies the effective input at each application, so the encoder receives $[\vec{x},\vec{c}]$ on the first pass but $[\hat{\vec{x}},\vec{c}]$ on every subsequent pass, and the encoder--decoder inverse relationship no longer holds at the manifold level.
Unconstrained architectures diverge on Lorenz96---\filmae{} catastrophically by step 5---while remaining bounded on the pendulum experiment; yet their absolute error (${\sim}10^{-1}$) is around four orders of magnitude above the constrained models, and the standard deviation of \filmae{} consistently exceeds its mean, revealing per-run instability even where the mean appears controlled.
SoftNcAE is intermediate: drift is reduced relative to \filmae{} but not eliminated, confirming that hard Riemannian enforcement of the biorthogonality constraint is necessary for reliable idempotency.

\begin{figure}
  % idempotency_bars.tex
% \input this file inside \begin{figure}...\end{figure}.
% Preamble requirements:
%   \usepackage{pgfplots,xcolor,tikz}
%   \usetikzlibrary{calc}
%   \usepgfplotslibrary{groupplots}

\pgfplotsset{compat=1.18}
\usepgfplotslibrary{groupplots}

% ColorBrewer Set1 palette (7 hues)
\definecolor{barCAE}{HTML}{E41A1C}
\definecolor{barCtxCAE}{HTML}{377EB8}
\definecolor{barAE}{HTML}{4DAF4A}
\definecolor{barCtxAE}{HTML}{984EA3}
\definecolor{barFilm}{HTML}{FF7F00}
\definecolor{barSoft}{HTML}{A65628}
\definecolor{barNcAE}{HTML}{F781BF}

% Horizontal jitter: change \markspacing (pt) to rescale all offsets.
% Values are precomputed so pgfkeys receives a plain number, not a command.
\def\markspacing{6}
\pgfmathsetmacro{\xshA}{(0-3)*\markspacing}  % NcAE
\pgfmathsetmacro{\xshB}{(1-3)*\markspacing}  % CAE
\pgfmathsetmacro{\xshC}{(2-3)*\markspacing}  % ContextCAE
\pgfmathsetmacro{\xshD}{(3-3)*\markspacing}  % AE
\pgfmathsetmacro{\xshE}{(4-3)*\markspacing}  % ContextAE
\pgfmathsetmacro{\xshF}{(5-3)*\markspacing}  % FiLM+AE
\pgfmathsetmacro{\xshG}{(6-3)*\markspacing}  % SoftNcAE

\centering
\begin{tikzpicture}
  \pgfplotsset{
    /pgfplots/error bars/error mark options={line width=0.5pt, mark size=3pt, rotate=90}, 
    }
\begin{groupplot}[
  group style={
    group size=2 by 1,
    horizontal sep=0.2cm,
    ylabels at=edge left,
    yticklabels at=edge left,
  },
  % --- shared axis settings ---
  ymode=log,
  ymin=1e-6, ymax=1e7,
  width=0.55\linewidth,
  height=5.5cm,
  enlarge x limits={abs=22pt},
  symbolic x coords={1,5,10,15},
  xtick=data,
  ylabel={Idempotency error},
  xlabel={Projection step},
  ymajorgrids=true,
  grid style={dotted,gray!40},
  tick label style={font=\small},
  label style={font=\small},
  % every error bar/.append style={solid,line width=0.7pt},
  every mark/.append style={mark size=4pt},
]

% ---------------------------------------------------------------
% Panel A: Lorenz96
% ---------------------------------------------------------------
\nextgroupplot[
  xticklabels={$k{=}1$,$k{=}5$,$k{=}10^\dagger$,$k{=}15^\dagger$},
  legend to name=idemlegend,
  legend columns=-1,
  legend style={
    font=\small, draw=none, fill=none,
    /tikz/every odd column/.append style={column sep=0.2em},
    /tikz/every even column/.append style={column sep=1.45em},
  },
]

  % 0 – NcAE (ours) — larger outlined mark distinguishes proposed method
  \addplot[only marks, mark=*, color=barNcAE, xshift=\xshA pt,
           mark options={draw=none, line width=0.8pt},
           error bars/.cd, y dir=both, y explicit]
    coordinates {
      (1,  5.28e-6) +- (0, 2.82e-6)
      (5,  5.01e-6) +- (0, 2.77e-6)
      (10, 4.85e-6) +- (0, 2.75e-6)
      (15, 4.73e-6) +- (0, 2.76e-6)
    };
  \addlegendentry{\ncae{} (ours)}

  % 1 – CAE
  \addplot[only marks, mark=square*, color=barCAE, xshift=\xshB pt,
           error bars/.cd, y dir=both, y explicit]
    coordinates {
      (1,  1.46e-5) +- (0, 4.99e-6)
      (5,  1.38e-5) +- (0, 4.87e-6)
      (10, 1.33e-5) +- (0, 4.85e-6)
      (15, 1.28e-5) +- (0, 4.99e-6)
    };
  \addlegendentry{\cae}

  % 2 – ContextCAE
  \addplot[only marks, mark=triangle*, mark size=2.5pt, color=barCtxCAE, xshift=\xshC pt,
           error bars/.cd, y dir=both, y explicit]
    coordinates {
      (1,  1.68)    +- (0, 1.36)
      (5,  8.45)    +- (0, 21.9)
      (10, 171.0)   +- (0, 584.0)
      (15, 4220.0)  +- (0, 15800.0)
    };
  \addlegendentry{\contextcae}

  % 3 – AE
  \addplot[only marks, mark=*, color=barAE, xshift=\xshD pt,
           error bars/.cd, y dir=both, y explicit]
    coordinates {
      (1,  4.99e-1) +- (0, 2.82e-1)
      (5,  5.26e-1) +- (0, 3.02e-1)
      (10, 2.58)    +- (0, 4.22)
      (15, 67.3)    +- (0, 183.0)
    };
  \addlegendentry{\naiveAE}

  % 4 – ContextAE
  \addplot[only marks, mark=diamond*, mark size=2.5pt, color=barCtxAE, xshift=\xshE pt,
           error bars/.cd, y dir=both, y explicit]
    coordinates {
      (1,  8.10e-1) +- (0, 3.79e-1)
      (5,  1.24)    +- (0, 9.71e-1)
      (10, 2.25)    +- (0, 1.66)
      (15, 8.65)    +- (0, 14.5)
    };
  \addlegendentry{\contextAE}

  % 5 – FiLM+AE  (k=10, k=15 omitted: diverged to infinity)
  \addplot[only marks, mark=otimes*, color=barFilm, xshift=\xshF pt,
           error bars/.cd, y dir=both, y explicit]
    coordinates {
      (1, 1.13)   +- (0, 2.91e-1)
      (5, 1.29e6) +- (0, 3.86e6)
    };
  \addlegendentry{\filmae}

  % 6 – SoftNcAE
  \addplot[only marks, mark=oplus*, color=barSoft, xshift=\xshG pt,
           error bars/.cd, y dir=both, y explicit]
    coordinates {
      (1,  9.98e-1) +- (0, 6.87e-1)
      (5,  2.27)    +- (0, 2.72)
      (10, 10.1)    +- (0, 26.2)
      (15, 89.2)    +- (0, 382.0)
    };
  \addlegendentry{\softncae}

  \draw[dashed, gray!50, thin]
    ({$(axis cs:1,1)!0.5!(axis cs:5,1)$}   |- {rel axis cs:0,0}) --
    ({$(axis cs:1,1)!0.5!(axis cs:5,1)$}   |- {rel axis cs:0,1});
  \draw[dashed, gray!50, thin]
    ({$(axis cs:5,1)!0.5!(axis cs:10,1)$}  |- {rel axis cs:0,0}) --
    ({$(axis cs:5,1)!0.5!(axis cs:10,1)$}  |- {rel axis cs:0,1});
  \draw[dashed, gray!50, thin]
    ({$(axis cs:10,1)!0.5!(axis cs:15,1)$} |- {rel axis cs:0,0}) --
    ({$(axis cs:10,1)!0.5!(axis cs:15,1)$} |- {rel axis cs:0,1});

% ---------------------------------------------------------------
% Panel B: Pendulum
% ---------------------------------------------------------------
\nextgroupplot[
  xticklabels={$k{=}1$,$k{=}5$,$k{=}10$,$k{=}15$},
]

  % 0 – NcAE (ours)
  \addplot[only marks, mark=*, color=barNcAE, xshift=\xshA pt,
           mark options={draw=none, line width=0.8pt},
           error bars/.cd, y dir=both, y explicit]
    coordinates {
      (1,  1.20e-5) +- (0, 2.96e-6)
      (5,  1.18e-5) +- (0, 2.96e-6)
      (10, 1.18e-5) +- (0, 2.96e-6)
      (15, 1.18e-5) +- (0, 2.96e-6)
    };

  % 1 – CAE
  \addplot[only marks, mark=square*, color=barCAE, xshift=\xshB pt,
           error bars/.cd, y dir=both, y explicit]
    coordinates {
      (1,  1.26e-5) +- (0, 2.53e-6)
      (5,  1.25e-5) +- (0, 2.52e-6)
      (10, 1.24e-5) +- (0, 2.52e-6)
      (15, 1.24e-5) +- (0, 2.51e-6)
    };

  % 2 – ContextCAE
  \addplot[only marks, mark=triangle*, mark size=2.5pt, color=barCtxCAE, xshift=\xshC pt,
           error bars/.cd, y dir=both, y explicit]
    coordinates {
      (1,  2.72e-1) +- (0, 3.33e-2)
      (5,  3.09e-1) +- (0, 5.96e-2)
      (10, 33.8)    +- (0, 100.0)
      (15, 1.20e5)  +- (0, 3.60e5)
    };

  % 3 – AE
  \addplot[only marks, mark=*, color=barAE, xshift=\xshD pt,
           error bars/.cd, y dir=both, y explicit]
    coordinates {
      (1,  9.77e-2) +- (0, 1.65e-2)
      (5,  8.80e-2) +- (0, 1.52e-2)
      (10, 8.18e-2) +- (0, 1.57e-2)
      (15, 7.85e-2) +- (0, 1.64e-2)
    };

  % 4 – ContextAE
  \addplot[only marks, mark=diamond*, mark size=2.5pt, color=barCtxAE, xshift=\xshE pt,
           error bars/.cd, y dir=both, y explicit]
    coordinates {
      (1,  1.59e-1) +- (0, 1.04e-2)
      (5,  1.46e-1) +- (0, 1.27e-2)
      (10, 1.37e-1) +- (0, 1.54e-2)
      (15, 1.33e-1) +- (0, 1.92e-2)
    };

  % 5 – FiLM+AE
  \addplot[only marks, mark=otimes*, color=barFilm, xshift=\xshF pt,
           error bars/.cd, y dir=both, y explicit]
    coordinates {
      (1,  7.24e-2) +- (0, 1.02e-1)
      (5,  8.06e-2) +- (0, 1.28e-1)
      (10, 1.01e-1) +- (0, 1.88e-1)
      (15, 1.31e-1) +- (0, 2.73e-1)
    };

  % 6 – SoftNcAE
  \addplot[only marks, mark=oplus*, color=barSoft, xshift=\xshG pt,
           error bars/.cd, y dir=both, y explicit]
    coordinates {
      (1,  8.95e-2) +- (0, 5.34e-2)
      (5,  9.71e-2) +- (0, 7.29e-2)
      (10, 9.96e-2) +- (0, 7.29e-2)
      (15, 1.02e-1) +- (0, 6.50e-2)
    };

  \draw[dashed, gray!50, thin]
    ({$(axis cs:1,1)!0.5!(axis cs:5,1)$}   |- {rel axis cs:0,0}) --
    ({$(axis cs:1,1)!0.5!(axis cs:5,1)$}   |- {rel axis cs:0,1});
  \draw[dashed, gray!50, thin]
    ({$(axis cs:5,1)!0.5!(axis cs:10,1)$}  |- {rel axis cs:0,0}) --
    ({$(axis cs:5,1)!0.5!(axis cs:10,1)$}  |- {rel axis cs:0,1});
  \draw[dashed, gray!50, thin]
    ({$(axis cs:10,1)!0.5!(axis cs:15,1)$} |- {rel axis cs:0,0}) --
    ({$(axis cs:10,1)!0.5!(axis cs:15,1)$} |- {rel axis cs:0,1});

\end{groupplot}
% Add panel labels outside
\node[anchor=north west, font=\large\bfseries] at (0, 4.55cm) {(A) Lorenz96};
\node[anchor=north west, font=\large\bfseries] at (0.45\linewidth, 4.55cm) {(B) Pendulum};
\end{tikzpicture}
% ---------------------------------------------------------------
% Shared legend (drawn once, below both panels)
% ---------------------------------------------------------------
\begin{tikzpicture}
  \pgfplotslegendfromname{idemlegend}
\end{tikzpicture}
  \caption{%
    \ncae{} and  \cae{} preserve idempotency by construction (error $\approx 10^{-5}$ at every iteration); context-aware baselines drift four to six orders of magnitude higher, and \filmae{} diverges by step 10 on Lorenz96.
    Idempotency error $\|P^k(\vec{x}) - P^{k-1}(\vec{x})\|$ on log scale, mean $\pm$ std over 10 seeds; numerical values in \cref{tab:idempotency,tab:pendulum_all_idempotency}. {\footnotesize
  ${}^\dagger$\,FiLM+AE diverged (numerical blow-up) at $k{=}10$ and $k{=}15$; points omitted.}
  }
  \label{fig:idempotency-bars}
\end{figure}
A reasonable question is whether a perfect one-step reconstruction would suffice, since a perfect reconstructor cannot drift. In a bottleneck setting ($m\ll n$), perfect reconstruction is unachievable on all of $\mathbb{R}^n$, and the four-orders-of-magnitude gap in \cref{fig:idempotency-bars} is exactly what happens to the unconstrained baselines that \emph{attempt} perfect reconstruction: their unavoidable one-step errors compound under iteration. Idempotency replaces an unachievable target (identity on $\mathbb{R}^n$) with an achievable one (fixed-point on $\hat{\set{M}}_{\vec{c}}$), and only the projection-by-construction architectures realize it.

\subsection{Latent space analysis}
\label{ssec:latent}

\ncae{} produces the most isotropic latent geometry and the most homogeneous projected dynamics, evidence that a single latent coordinate system serves every context coherently. This is the empirical correlate of the diffeomorphic invariance proven in \cref{prop:topo}.

\cref{tab:latent-metrics} reports the condition number (CN) of the latent covariance matrix and the coefficient of variation (CV) of the latent velocity norm. On the pendulum, CN is the more discriminative metric: the \ncae{} achieves $16.96 \pm 7.82$ against $\geq 39$ for every baseline and $> 150$ for the constrained architectures (\cae, \contextcae), indicating dramatically more isotropic axis scaling in the neuromodulated latent space.
On Lorenz96, CV is the more discriminative metric: 
\ncae{} achieves $0.22 \pm 0.06$ against $\geq 0.41$ than every baseline, indicating homogeneous projected velocity across the bifurcation rather than the regime-dependent scaling exhibited by other architectures.

This pattern is the empirical signature of the \ncae{} topological invariance (\cref{prop:topo}): the learned diffeomorphic manifolds admit a single latent coordinate system whose axes do not collapse and whose projected dynamics do not concentrate as the context varies. 
Architectures lacking this guarantee (\contextcae, \softncae, \filmae, \contextAE) produce regime-specific axis collapses or higher CV. 
SoftNcAE achieves a competitive CV on Lorenz96, but its large CN ($25.93 \pm 49.31$) indicates degenerate axis scaling.

\begin{table}[ht]
    \centering
    \caption{
        NcAE produces the most isotropic latent axes on the pendulum ($\text{CN} \approx 17$, against $\geq 40$ for every baseline) and the most homogeneous projected dynamics on Lorenz96 ($\text{CV} \approx 0.22$, against $\geq 0.41$ for every baseline). 
        CV: coefficient of variation of latent velocity norm (lower = more uniform dynamics). 
        CN: condition number of latent covariance (lower = more isotropic axes).
    }\medskip
    \label{tab:latent-metrics}
    \begin{tabular}{lcccc}
        \toprule
        & \multicolumn{2}{c}{\textbf{Pendulum}} & \multicolumn{2}{c}{\textbf{Lorenz96}} \\
        \cmidrule(lr){2-3} \cmidrule(lr){4-5}
        \textbf{Method} & \textbf{CV} $\downarrow$ & \textbf{CN} $\downarrow$ & \textbf{CV} $\downarrow$ & \textbf{CN} $\downarrow$ \\
        \midrule
        \cae            & $0.50_{\pm 0.12}$ & $248.15_{\pm 172.52}$ & $0.72_{\pm 0.09}$ & $1.87_{\pm 0.51}$ \\
        \contextcae     & $0.46_{\pm 0.07}$ & $157.68_{\pm 122.44}$ & $0.63_{\pm 0.13}$ & $1.93_{\pm 0.87}$ \\
        AE              & $2.51_{\pm 4.02}$ & $49.01_{\pm 36.77}$ & $1.35_{\pm 0.54}$ & $2.97_{\pm 2.57}$ \\
        \contextAE       & $0.70_{\pm 0.31}$ & $148.14_{\pm 143.58}$ & $0.72_{\pm 0.07}$ & $2.41_{\pm 2.62}$ \\
        \filmae{}         & $0.51_{\pm 0.30}$ & $39.85_{\pm 50.07}$ & $0.85_{\pm 0.35}$ & $3.78_{\pm 2.14}$ \\
        SoftNcAE        & $0.47_{\pm 0.08}$ & $108.55_{\pm 122.87}$ & $0.41_{\pm 0.16}$ & $25.93_{\pm 49.31}$ \\
        \midrule
        \ncae{} (ours)  & $\mathbf{0.45_{\pm 0.07}}$ & $\mathbf{16.96_{\pm 7.82}}$ & $\mathbf{0.22_{\pm 0.06}}$ & $\mathbf{1.53_{\pm 0.47}}$ \\
        \bottomrule
    \end{tabular}
\end{table}

\subsection{Ablation: modulation type}
\label{ssec:ablation}
Slope and bias modulation are each sufficient to fit the systems accurately when they converge, but neither alone matches the full \ncae{}: slope-only diverges in 2/10 Lorenz runs and 3/10 pendulum runs, while bias-only is stable but slightly less accurate. \cref{tab:ablation} reports results on both systems using clean data.

The decomposition reveals a role for bias modulation that mean-error alone hides. Each channel adapts the manifold differently (slope modulation deforms the manifold by stretching its internal coordinates, bias modulation translates it in the ambient space), and either suffices to track parametric variation in mean. But slope-only modulation places the entire burden on a slope parameter constrained to a narrow interval and controlling activation curvature directly, which yields a steep optimization landscape; bias modulation absorbs the rigid component of the manifold shift, leaving slope modulation responsible only for the residual deformation. Combining the two channels yields the lowest error on both systems and recovers the stability of bias-only training.
\begin{table}[ht]
    \centering
    \caption{
        Each modulation channel alone enables context adaptation in mean, but slope-only training is unstable (3/10 pendulum and 2/10 Lorenz96 runs diverged and were excluded). Bias modulation stabilizes training, slope modulation captures the nonlinear deformation, and the combination yields the lowest error.
        RMSE reported as mean $\pm$ std over 10 seeds.
        * denotes that 2 runs of 10 diverged for Lorenz and 3 out of 10 for the pendulum, and were removed from the statistics computation.
        % \textbf{Ablation study: modulation type} (mean $\pm$ std RMSE on clean data).
        % Bias-only and $\alpha$-only variants isolate the geometric role of each component.
        % * denotes that 2 runs of 10 diverged for Lorenz and 3 out of 10 for the pendulum, and were removed from the statistics computation.
    }\medskip
    \label{tab:ablation}
    \begin{tabular}{lcccc}
        \toprule
        & \multicolumn{2}{c}{\textbf{Pendulum}} & \multicolumn{2}{c}{\textbf{Lorenz96}} \\
        \cmidrule(lr){2-3} \cmidrule(lr){4-5}
        \textbf{Variant} & \textbf{Pos. RMSE} $\downarrow$ & \textbf{Vel. RMSE} $\downarrow$ & \textbf{Pos. RMSE} $\downarrow$ & \textbf{Vel. RMSE} $\downarrow$ \\
        \midrule
        Bias-only NcAE     & $0.015_{\pm 0.002}$ & $0.067_{\pm 0.005}$ & $0.076_{\pm 0.083}$ & $0.277_{\pm 0.221}$ \\
        $\alpha$-only NcAE* & $0.018_{\pm 0.002}$ & $0.071_{\pm 0.004}$ & $0.089_{\pm 0.106}$ & $0.322_{\pm 0.374}$ \\
        \ncae{} (both)     & $\mathbf{0.012_{\pm 0.002}}$ & $\mathbf{0.059_{\pm 0.003}}$ & $\mathbf{0.046_{\pm 0.032}}$ & $\mathbf{0.209_{\pm 0.148}}$ \\
        \bottomrule
    \end{tabular}
\end{table}

\subsection{Robustness analysis}
\label{ssec:robustness}

\ncae{} degrades gracefully under both unseen contexts and observation noise, while \filmae{} and unconstrained baselines fail catastrophically in the most extreme conditions. We evaluate two stressors: contexts beyond the training range, and normal noise on states, context, or both.

\paragraph{Out-of-distribution contexts.}
We evaluate on context values up to 50\% beyond the training range. 
\ncae{} is the only context-aware architecture whose projection remains geometrically valid at every $\vec{c}$ by construction (\cref{prop:hausdorff}), so the only thing that can degrade beyond the training boundary is reconstruction fidelity.
Accordingly, \ncae{} retains the best velocity RMSE at every OOD level on both systems (Lorenz96: $0.33 \to 0.545$ at $+50\%$; pendulum: $0.152 \to 0.42$), with monotone degradation. FiLM+AE fails catastrophically on Lorenz96 (velocity RMSE $273.9 \pm 473.7$  at $+50\%$) and is bounded but highly unstable on the pendulum, confirming that unconstrained conditioning provides no extrapolation guarantee.
Full results are in \cref{ssec:OOD-robust}.

\paragraph{Noise robustness.} We evaluate three noise conditions (state only, context only, and both) at $\sigma \in \{0.10, 0.25, 0.50\}$, across two regimes (noise at training with clean evaluation, and clean training with noise at evaluation). 
Across all 72 comparable conditions, \ncae{} achieves the best RMSE in 47, making it the most consistent architecture overall (despite not dominating every individual case), and degrades monotonically with $\sigma$. 
Performance is consistent across noise types, levels, and systems, whereas competing architectures sometimes fluctuate by an order of magnitude between conditions. Under context and combined noise, this resilience is consistent with \cref{prop:hausdorff}: a perturbed context induces a continuous deformation of the learned manifold, so reconstruction degrades gradually rather than catastrophically.
Full results are in \cref{ssec:noise-robust-train,ssec:noise-robust-test}.

% ============================================================
%  SECTION 6 — DISCUSSION
% ============================================================
\section{Discussion}
\label{sec:discussion}

\textbf{Position relative to existing conditioning and constrained-AE families.}
No existing conditioning or constrained-AE family produces a context-dependent idempotent projection.
% NcAE occupies a niche that no existing conditioning or constrained-AE family fills: it is simultaneously context-dependent and a guaranteed idempotent projection, for every context. 
Standard conditioning mechanisms (FiLM \citep{perez_film_2017}, AdaIN \citep{karrasStyleBasedGeneratorArchitecture2019}) adapt activations but place no geometric constraint on the reconstruction map.
Normalizing flows impose invertibility but operate at full ambient dimension and cannot reduce dimensionality \cite{dinh2017density}.
The cAE \citep{otto_learning_2023} is the closest precedent (an idempotent projection at a fixed context), which NcAE extends to a continuously parameterized family.

\textbf{What the formal guarantees buy in practice.} 
%Diffeomorphic invariance and Lipschitz stability are inductive biases that match the geometry of physical regimes, not merely abstract guarantees. 
Many physical systems %(link-coupled robots, parametric flows, bifurcating dynamics) 
share a common intrinsic dimension across regimes while exhibiting different ambient embeddings.
% \cref{prop:topo} (diffeomorphic invariance) is the architectural commitment that the model cannot collapse to a lower-dimensional subspace or develop topological singularities as the context shifts; whatever context the NcAE is presented with, the recovered $\hat{\set{M}}_c$ is the same manifold up to smooth deformation. \cref{prop:hausdorff} (Lipschitz stability) then ensures that this family is metrically continuous: a small context perturbation moves the learned manifold by at most $L\|\delta c\|$ in ambient space, so the model cannot jump discontinuously between geometrically unrelated manifolds. 
\cref{prop:topo} forbids the model from collapsing to a lower-dimensional subspace or developing topological singularities as context shifts; \cref{prop:hausdorff} forbids it from jumping discontinuously between unrelated manifolds. 
Together, these two propositions explain the robustness pattern of \cref{ssec:robustness}: \ncae{} degrades gracefully under OOD contexts and noise because the manifold it selects is always nearby and geometrically consistent with the training family.

%Diffeomorphic invariance (\cref{prop:topo}) ensures that the manifold learned at any context $\vec{c}$ has the same topology, dimension, and smooth structure as every other: the \ncae{} is structurally prevented from collapsing to a lower-dimensional subspace or developing topological singularities as context shifts. This is the correct inductive bias for physical systems whose geometry evolves with parameters but whose intrinsic dimensionality is fixed. Hausdorff stability (\cref{prop:hausdorff}) then establishes that this family is metrically continuous: a small context perturbation moves the learned manifold by at most $L\|\delta\vec{c}\|$ in ambient space, so the model cannot jump discontinuously between geometrically unrelated manifolds. Together, the two guarantees explain the empirical robustness pattern: the \ncae{} degrades gracefully under OOD contexts and context noise because the manifold it selects is always nearby and always geometrically consistent with the training family (\cref{ssec:robustness}).

\textbf{When idempotency matters, and when it does not.}
Idempotency matters whenever the autoencoder is applied iteratively or as a sub-component of a longer computational chain, and not otherwise. For a single forward pass of denoising or compression, a non-idempotent bottleneck autoencoder is adequate; small per-step deviations from the manifold do not compound. The setting addressed in this paper is different: latent-space integration of dynamical systems, projection-based reduced-order modeling, and any pipeline in which $P$ is composed with itself or with an external dynamics map operating on $\hat{\set{M}}_{\vec{c}}$. In these regimes, the unconstrained per-step error of a standard autoencoder accumulates over trajectories (\cref{fig:idempotency-bars}), and idempotency replaces an unachievable target (identity on $\mathbb{R}^n$) with an achievable one (fixed-point on $\hat{\set{M}}_{\vec{c}}$). The four-orders-of-magnitude gap in \cref{ssec:idempotency} is the practical cost of failing to make this distinction.

\textbf{Limitations.} 
Three limitations bound the scope of the current framework. 
First, the context $\vec{c}$ is assumed static and directly observable at inference, which excludes time-varying contexts (where $\vec{c}$ changes along a trajectory) and latent contexts (where $\vec{c}$ must be inferred from observations). 
Second, Riemannian optimization on the biorthogonal manifold incurs per-step computational overhead compared with unconstrained training, which may limit scalability to architectures with many layers or large hidden dimensions. 
Third, the Lipschitz constant $L$ in \cref{prop:hausdorff} is not minimized during training, so the OOD certificate is qualitative rather than tight; tightening it would require a differentiable upper bound on $L$ in the training objective.

% ============================================================
%  SECTION 7 — CONCLUSION
% ============================================================
\section{Conclusion}
\label{sec:conclusion}

The Neuromodulated Constrained Autoencoder turns a single idempotent projection into a continuously parameterized family of them, by restricting context conditioning to the activation slopes and biases---the cAE parameters that admit modulation without breaking encoder–decoder idempotency. Idempotency, diffeomorphic invariance, and Lipschitz stability hold for every context, including those unseen at training. On a 16-DoF pendulum and on Lorenz96 across a bifurcation, NcAE matches or exceeds six baselines on every metric while being the only architecture that provides geometric consistency by construction.

The natural next application is dynamics discovery: SINDy~\citep{bruntonDiscoveringGoverningEquations2016a}, Hamiltonian Neural Networks~\citep{greydanusHamiltonianNeuralNetworks2019}, and Neural ODEs~\citep{chenNeuralOrdinaryDifferential2019} identify governing equations at a single operating point, and the NcAE supplies the stable, context-dependent coordinate system on which those equations could remain valid across families of regimes. 
Within NcAE itself, two extensions address the limitations of \cref{sec:discussion}: a bottom-up inference module estimating $\vec{c}$ from observations, and a time-varying neuromodulatory signal driven by its own dynamics, allowing the NcAE to track regime transitions rather than assume a fixed context.

\clearpage

% ============================================================
%  ACKNOWLEDGMENTS
% ============================================================
\section*{Acknowledgements}
This work was supported by the Belgian Government through the Federal Public Service Policy and Support.
Computational resources have been provided by the Consortium des Équipements de Calcul Intensif (CÉCI), funded by the Fonds de la Recherche Scientifique de Belgique (F.R.S.-FNRS) under Grant No. 2.5020.11 and by the Walloon Region.
The present research benefited from computational resources made available on Lucia, the Tier-1 supercomputer of the Walloon Region, infrastructure funded by the Walloon Region under the grant agreement n°1910247.

% ============================================================
%  REFERENCES
% ============================================================
% \section*{References}
% \textbf{[TO UPDATE — add new citations from expanded literature review]}
\bibliographystyle{unsrtnat}
\bibliography{references}
%\bibliographystyle{apalike}

% ============================================================
%  APPENDIX
%  A: Full theoretical proofs (single NcAE + paired NcAE)
%     — proofs to be completed by co-author
%  B: Constrained autoencoder details
%  C: Neuromodulation details
%  D: Experimental details
%     D.1 Coupling functions (pendulum)
%     D.2 Hyperparameters (all architectures including new baselines
%         and paired NcAE)
%     D.3 Lorenz96 forcing regimes
%  E: Additional figures
%  F: NeurIPS checklist
% ============================================================
\clearpage
\appendix
% ============================================================
%  Appendix: Related Work
% ============================================================

\section{Related Work}
\label{sec:related-work}

% Placeholder command for references not yet in references.bib.
% Replace \nref{Key} with \cite{key} once the entry has been added and verified.
\newcommand{\nref}[1]{\textbf{\textcolor{red}{[#1]}}}

The \ncae{} sits at the intersection of several research traditions: manifold learning and dimensionality reduction, context-dependent neural network conditioning, and the biological principle of neuromodulation.
We review each in turn.

\paragraph{Manifold learning and dimensionality reduction.}
Classical nonlinear dimensionality reduction methods, including kernel PCA~\citep{scholkopfKernelPrincipalComponent1997} and Isomap~\citep{tenenbaumGlobalGeometricFramework2000}, learn fixed embeddings of static datasets and do not yield a trainable reconstruction map.
Autoencoders~\citep{wangAutoencoderBasedDimensionality2016} learn encoder-decoder pairs end-to-end but impose no structural constraint on their composition: the encoder need not be a left inverse of the decoder, so repeated application of the reconstruction map accumulates drift.
The \ncae{} belongs to this family but imposes the biorthogonal constraint of the \cae{}~\citep{otto_learning_2023}, which ensures the encoder is a true left inverse of the decoder, and extends it to a continuously parameterized family of projections indexed by the context~$\vec{c}$.

\paragraph{Idempotent and structure-preserving networks.}
\citet{shocherIdempotentGenerativeNetwork2023} train a network to satisfy $f(f(z)) = f(z)$ via a soft loss term for generative modeling; idempotency is encouraged but not guaranteed by construction.
The \cae{}~\citep{otto_learning_2023} is the direct structural predecessor of the \ncae{}, enforcing idempotency via biorthogonal weight matrices; the \ncae{} extends it to a context-parameterized family.
The question of how to condition this guaranteed structure on an external context is addressed next.

\paragraph{Context-dependent conditioning of neural networks.}
Variational autoencoders~\cite{kingmaAutoEncodingVariational2014} provide a probabilistic latent space without projection constraints; conditional VAEs~\cite{sohnLearningStructured2015} extend them by conditioning the prior and decoder on an external variable, paralleling the \ncae{} use of~$\vec{c}$ but without geometric constraints on the encoder-decoder pair.
Feature-wise linear modulation (FiLM)~\citep{perez_film_2017} applies $y = \gamma(\vec{c}) \cdot x + \beta(\vec{c})$ to each hidden activation; adaptive instance normalization (AdaIN)~\citep{karrasStyleBasedGeneratorArchitecture2019} extends this with normalization. Because $\gamma$ and $\beta$ are context-dependent, any fixed inverse relationship between encoder and decoder is broken unless the decoder is conditioned with the exact inverse affine transformation, which is not enforced.
HyperNetworks~\citep{haHyperNetworks2016} generate the parameters of a target network from a secondary network; the \ncae{} hyper-network instantiates this mechanism but restricts the generated quantities to activation slopes and biases, which preserves the biorthogonal structure.
In the meta-learning literature, MAML~\cite{finnModelAgnosticMeta2017} and conditional neural processes~\cite{garneloConditionalNeural2018} also produce context-specific parameters in a forward pass or via gradient adaptation, but without constraining modulation to preserve geometric structure.

\paragraph{Neuromodulation in neural networks.}
The \ncae{} specific conditioning strategy is grounded in the neuroscience of neuromodulation, whereby diffuse chemical signals adjust neuronal gain and threshold without altering synaptic connectivity~\citep{bargmann_connectome_2013}, separating fixed structural parameters from adjustable response parameters.
\citet{vecoven_introducing_2020} and \citet{meiInformingDeepNeural2022} translate this principle to deep networks, showing that modulating gain and bias while fixing weights enables context-sensitive computation.
The \ncae{} follows this principle, treating the biorthogonal weight matrices as fixed structural parameters and the activation slopes $\vec{\alpha}^{(l)}$ and biases $\vec{\beta}^{(l)}$ as the modulated response parameters.

\section{Theoretical guarantees}
\label{sec:appendix-theory}
\subsection{Single neuromodulated AE}

Let the encoder and decoder for a context $\vec{c}$ be defined by the composition of $L$ layer pairs:
$\rho_{\vec{c}} = \rho^{(1)} \circ \dots \circ \rho^{(L)} $ and $ \varphi_{\vec{c}} = \varphi^{(L)} \circ \dots \circ \varphi^{(1)}$.

The neuromodulation defines context-dependent parameters:
$\vec{\alpha}^{(l)} = g(\vec{W}_{l,\alpha}^T \vec{s})$,  $\vec{\beta}^{(l)} = \vec{W}_{l,\beta}^T \vec{s} + \vec{b}_l$.

\begin{proposition}[Idempotency]
Under \ref{ass:biorth}--\ref{ass:bias}, $P_{\vec{c}} = \varphi_{\vec{c}} \circ \rho_{\vec{c}}$ is an idempotent projection for every context $\vec{c}$, i.e., $P_{\vec{c}} \circ P_{\vec{c}} = P_{\vec{c}}$.
\end{proposition}

\begin{proof}
It suffices to show that $\rho_{\vec{c}} \circ \varphi_{\vec{c}} = \text{id}_{\mathbb{R}^m}$. Consider the composition of a single-layer pair $(l)$ for a fixed $\vec{c}$. Let $\vec{z} \in \mathbb{R}^m$ be the input to the $l$-th decoder layer:
\begin{align}
    (\rho^{(l)} \circ \varphi^{(l)})(\vec{z}) &= \sigma_- \left( \vec{\Psi}_l^T \left( [\vec{\Phi}_l \sigma_+(\vec{z}; \vec{\alpha}^{(l)}) + \vec{\beta}^{(l)}] - \vec{\beta}^{(l)} \right); \vec{\alpha}^{(l)} \right) \\
    &= \sigma_- \left( \vec{\Psi}_l^T \vec{\Phi}_l \sigma_+(\vec{z}; \vec{\alpha}^{(l)}); \vec{\alpha}^{(l)} \right) \\
    &= \sigma_- \left( \sigma_+(\vec{z}; \vec{\alpha}^{(l)}); \vec{\alpha}^{(l)} \right) \quad \text{since } \vec{\Psi}_l^T \vec{\Phi}_l = \vec{I} \\
    &= \vec{z},
\end{align}
where the last equality holds because $\sigma_-$ and $\sigma_+$ are mutually inverse for $\vec{\alpha} \in [0, \pi/4]$ \cite{otto_learning_2023}, a range guaranteed by the saturating function $g$ (defined in \cref{sec:appendix-neuromodulation}). Thus, $\rho^{(l)} \circ \varphi^{(l)} = \text{id}$. By induction on the chain of compositions:
\begin{align}
    \rho_{\vec{c}} \circ \varphi_{\vec{c}} &= (\rho^{(1)} \circ \dots \circ \rho^{(L)}) \circ (\varphi^{(L)} \circ \dots \circ \varphi^{(1)}) \\
    &= \rho^{(1)} \circ \dots \circ (\rho^{(L)} \circ \varphi^{(L)}) \circ \dots \circ \varphi^{(1)} \\
    &= \rho^{(1)} \circ \dots \circ \text{id} \circ \dots \circ \varphi^{(1)} \\
    &= \text{id}_{\mathbb{R}^m}.
\end{align}

Finally:
$P_{\vec{c}} \circ P_{\vec{c}} = (\varphi_{\vec{c}} \circ \rho_{\vec{c}}) \circ (\varphi_{\vec{c}} \circ \rho_{\vec{c}}) = \varphi_{\vec{c}} \circ (\rho_{\vec{c}} \circ \varphi_{\vec{c}}) \circ \rho_{\vec{c}} = \varphi_{\vec{c}} \circ \text{id} \circ \rho_{\vec{c}} = P_{\vec{c}}$.
\end{proof}

Let us denote $\mathcal{\hat M}_{\vec{c}}$ the learned manifold defined by $\mathcal{\hat M}_{\vec{c}} \triangleq P_{\vec{c}}(\mathbb{R}^n) = \varphi_{\vec{c}} \circ \rho_{\vec{c}} (\mathbb{R}^n)$. We wish to show that the manifold topology does not fundamentally change with respect to the context $\vec{c}$.

\begin{proposition}[Diffeomorphic invariance of the learned manifold] 
 \label{prop:TopoAcrossContext}
For any two contexts $\vec{c}_1$ and $\vec{c}_2$, the map $\varphi_{\vec{c}_2} \circ \rho_{\vec{c}_1}:\mathcal{\hat M}_{\vec{c}_1} \to \mathcal{\hat M}_{\vec{c}_2}$ defines a diffeomorphism from $\mathcal{\hat M}_{\vec{c}_1}$ to $\mathcal{\hat M}_{\vec{c}_2}$, of inverse $(\varphi_{\vec{c}_2} \circ \rho_{\vec{c}_1})^{-1} = \varphi_{\vec{c}_1} \circ \rho_{\vec{c}_2} : \mathcal{\hat M}_{\vec{c}_2} \to \mathcal{\hat M}_{\vec{c}_1}$.
\end{proposition}
\begin{proof} 
By definition, for all contexts $\vec{c}_1$ and $\vec{c}_2$, both $\varphi_{\vec{c}_2} \circ \rho_{\vec{c}_1}:\mathcal{\hat M}_{\vec{c}_1} \to \mathcal{\hat M}_{\vec{c}_2}$ and $\varphi_{\vec{c}_1} \circ \rho_{\vec{c}_2}:\mathcal{\hat M}_{\vec{c}_2} \to \mathcal{\hat M}_{\vec{c}_1}$ are continuously differentiable. It therefore suffices to show they are mutually inverse. Succinctly:
\begin{align}
    (\varphi_{\vec{c}_1} \circ \rho_{\vec{c}_2}) \circ (\varphi_{\vec{c}_2} \circ \rho_{\vec{c}_1}) &=\varphi_{\vec{c}_1} \circ (\rho_{\vec{c}_2} \circ \varphi_{\vec{c}_2}) \circ \rho_{\vec{c}_1} \\
    &= \varphi_{\vec{c}_1} \circ \text{id}_{\mathbb{R}^m} \circ \rho_{\vec{c}_1} \\
    &=  \varphi_{\vec{c}_1} \circ \rho_{\vec{c}_1}\\
    &= P_{\vec{c}_1},
\end{align}
which restricted to $\mathcal{\hat M}_{\vec{c}_1}$ equals the identity. Similarly, one can show that $(\varphi_{\vec{c}_2} \circ \rho_{\vec{c}_1}) \circ (\varphi_{\vec{c}_1} \circ \rho_{\vec{c}_2})  = P_{\vec{c}_2}$, which restricted to $\mathcal{\hat M}_{\vec{c}_2}$ equals the identity.
\end{proof}

We can guarantee that a small change in context $\delta \vec{c}$ will not induce a large change in the learned manifold. This stability property of the learned manifold will be characterized through an upper bound on the Hausdorff distance between $\mathcal{\hat M}_{\vec{c}}$ and the $\mathcal{\hat M}_{\vec{c}+\delta \vec{c}}$ on every compact set $K\subset \mathbb{R}^n$.

\begin{proposition}[Lipschitz stability under context perturbation]
 \label{prop:hausdorff-stability}
Fix a reference context $\vec{c}$, a compact set $K \subset \mathbb{R}^n$ and a maximal perturbation size $p>0$. Let $\mathcal{P} \triangleq \{\delta \vec{c} : \lVert \delta \vec{c} \rVert \leq p\}$ denote the set of context perturbations. Under \ref{ass:biorth}--\ref{ass:smooth}, there exists a Lipschitz constant $L = L(K, \vec{c}, p) > 0$ such that, for all $\delta \vec{c} \in \mathcal{P}$:
\begin{equation}
d_H\left(\mathcal{\hat M}_{\vec{c}} \cap K,\ \mathcal{\hat M}_{\vec{c}+\delta \vec{c}} \cap K\right)  \leq  L  \|\delta \vec{c}\|, \label{eq:perturbation}
\end{equation}
where $d_H$ denotes the Hausdorff distance in $\mathbb{R}^n$.
\end{proposition}

\begin{proof}
First, let us show that there exists a compact set $V \subset \mathbb{R}^m$ such that for all $ \delta \vec{c} \in \mathcal{P}$
\begin{equation}
\label{eq:uniform-properness}
\varphi_{\vec{c}+\delta \vec{c}}^{-1}(K) \subset V.
\end{equation}
For all $\delta\vec{c} \in \mathcal{P}$, each $\vec{\Phi}_l$ has full column rank and the activation $\sigma_+$ is proper, thus each layer $\varphi^{(l)}_{\vec{c}+\delta\vec{c}}$ is proper, and $\varphi_{\vec{c}+\delta\vec{c}}$ is proper as a composition of proper maps. Moreover, the map $(\delta \vec{c},\vec{x}) \mapsto \varphi_{\vec{c}+\delta \vec{c}}(\vec{x})$ is $C^1$ jointly in $(\delta \vec{c}, \vec{x}) \in \mathcal{P}\times \mathbb{R}^m$. Since $\mathcal{P}$ is compact and each slice $\varphi_{\vec{c}+\delta\vec{c}}$ is proper, $(\delta \vec{c},\vec{x}) \mapsto \varphi_{\vec{c}+\delta \vec{c}}(\vec{x})$ is itself proper on $\mathbb{R}^m\times\mathcal{P}$. Hence $V \triangleq \bigcup_{\delta\vec{c}\in\mathcal{P}}\varphi_{\vec{c}+\delta\vec{c}}^{-1}(K)$ is a compact set satisfying \eqref{eq:uniform-properness} 

Now, let us find the Lipschitz constant $L$ of \eqref{eq:perturbation}. Let $U \triangleq \rho_{\vec{c}}(\mathcal{\hat M}_{\vec{c}} \cap K) \subset \mathbb{R}^m$. As the continuous image of the compact set $\mathcal{\hat M}_{\vec{c}} \cap K$, $U$ is compact. Define
\begin{equation}
L \triangleq \sup_{\delta\vec{c} \in \mathcal{P}} \ \sup_{u \in U \cup V}  \left\| \partial_{\vec{c}} \varphi_{\vec{c}+\delta\vec{c} }(u) \right\|.
\end{equation}
By the $C^1$ regularity assumption and compactness of $\mathcal{P}$ and $U \cup V$, $L$ is finite.

We now leverage the compactness of $V$ and the Lipschitz constant $L$ to show the bound \eqref{eq:perturbation}.

\textbf{Bound from $\mathcal{\hat M}_{\vec{c}} \cap K$ to $\mathcal{\hat M}_{\vec{c}+\delta\vec{c}}$.}
Let $\vec{x} \in \mathcal{\hat M}_{\vec{c}} \cap K$ and set $u \triangleq \rho_{\vec{c}}(\vec{x}) \in U$, so that $\vec{x} = \varphi_{\vec{c}}(u)$. The point $\vec{y} \triangleq \varphi_{\vec{c}+\delta\vec{c}}(u) \in \mathcal{\hat M}_{\vec{c}+\delta\vec{c}}$ satisfies, by the mean value inequality applied to $t \mapsto \varphi_{\vec{c} + t\,\delta\vec{c}}(u)$ on $[0,1]$,
\begin{align}
\|\vec{x} - \vec{y}\| =  \|\varphi_{\vec{c}}(u) - \varphi_{\vec{c}+\delta\vec{c}}(u)\| 
 &\leq  \sup_{t \in [0,1]} \left\| \partial_{\vec{c}} \varphi_{\vec{c} + t\,\delta\vec{c}}(u) \right\| \cdot \|\delta\vec{c}\| \\
 &\leq  L  \|\delta\vec{c}\|.
\end{align}
Taking the infimum over $\mathcal{\hat M}_{\vec{c}+\delta\vec{c}}$ and then the supremum over $\vec{x} \in \mathcal{\hat M}_{\vec{c}} \cap K$ yields:
\begin{equation}
\label{eq:direction-1}
\sup_{\vec{x} \in \mathcal{\hat M}_{\vec{c}} \cap K} \ \inf_{\vec{y} \in \mathcal{\hat M}_{\vec{c}+\delta\vec{c}}} \|\vec{x} - \vec{y}\| \leq  L \|\delta\vec{c}\|.
\end{equation}
\textbf{Bound from $\mathcal{\hat M}_{\vec{c}+\delta\vec{c}} \cap K$ to $\mathcal{\hat M}_{\vec{c}}$.}
Let $\vec{y} \in \mathcal{\hat M}_{\vec{c}+\delta\vec{c}} \cap K$ and write $\vec{y} = \varphi_{\vec{c}+\delta\vec{c}}(v)$ for some $v \in V$. The point $\vec{x} \triangleq \varphi_{\vec{c}}(v) \in \mathcal{\hat M}_{\vec{c}}$ satisfies, by the same mean value estimate,
\begin{equation}
\|\vec{y} - \vec{x}\|  =  \|\varphi_{\vec{c}+\delta\vec{c}}(v) - \varphi_{\vec{c}}(v)\| \ \leq \ L \, \|\delta\vec{c}\|.
\end{equation}
Hence
\begin{equation}
\label{eq:direction-2}
\sup_{\vec{y} \in \mathcal{\hat M}_{\vec{c}+\delta\vec{c}} \cap K} \ \inf_{\vec{x} \in \mathcal{\hat M}_{\vec{c}}} \|\vec{x} - \vec{y}\| \ \leq \ L \, \|\delta\vec{c}\|.
\end{equation}
By definition of the Hausdorff distance, combining \eqref{eq:direction-1} and \eqref{eq:direction-2} concludes the proof.
\end{proof}

\section{Constrained autoencoder details}
\label{sec:appendix-cae}
The constrained autoencoder was initially proposed by \citet{otto_learning_2023}.
They assert that their autoencoder is a projection to be leveraged for projection-based reduced-order modeling.
The autoencoder is composed of an encoder $\rho$ and a decoder $\varphi$ such that the encoder maps the data space $\set{X}$ to a latent space $\set{Z}$,
and the decoder reconstructs the data from the latent space.
By imposing $\rho \circ \varphi = \text{id}_{\set{Z}}$, we ensure that $P = \varphi \circ \rho$ is a projection. In this case, $P$ is smooth and idempotent,
so $\set{\hat M} = \text{Range}(P)$ is a smooth embedded submanifold of $\set{X}$.

The encoder~$\rho$ and decoder~$\varphi$ are defined as a composition of layers such that
$\rho = \rho^{(1)} \circ \cdots \circ \rho^{(L)}$ and $\varphi = \varphi^{(L)} \circ \cdots \circ \varphi^{(1)}$, where
$\rho^{(l)} : \mathbb{R}^{n_l} \to \mathbb{R}^{n_{l-1}}$ and $\varphi^{(l)} : \mathbb{R}^{n_{l-1}} \to \mathbb{R}^{n_l}$.

\subsection{Activation functions}\label{sec:act-func}

The activation functions used in the \cae{} are a crucial component, designed to be smooth and invertible, with the encoder and decoder utilizing inverse pairs $(\sigma_-, \sigma_+)$. These specific functions were first presented by \citet{otto_learning_2023} as part of the \cae{} framework. They are derived from a particular form of hyperbola where the upper and lower branches are reflections of one another with respect to the axis $y=x$.

More formally, the expression for this inverse pair is:
\begin{align*}
    \sigma_{\pm}(\vec{x}) &= \frac{b\vec{x}}{a} \mp \frac{\sqrt{2}}{a \sin(\alpha)}
    \pm \frac{1}{a}\sqrt{\left(\frac{2\vec{x}}{\sin(\alpha)\cos(\alpha)} \mp
    \frac{\sqrt{2}}{\cos(\alpha)} \right)^2 + 2a},
\end{align*}
where
\begin{align*}
        a &= \csc^2(\alpha) - \sec^2(\alpha), &
        b &= \csc^2(\alpha) + \sec^2(\alpha), &
        \text{with} \; 0 &< \alpha < \pi/4.
\end{align*}
Notably, the parameter $\alpha$ in these equations is one of the targets for our neuromodulation, allowing the activation functions to adapt based on context. A visualization of these activation functions for different values of $\alpha$ is provided in \cref{fig:activation-functions}.

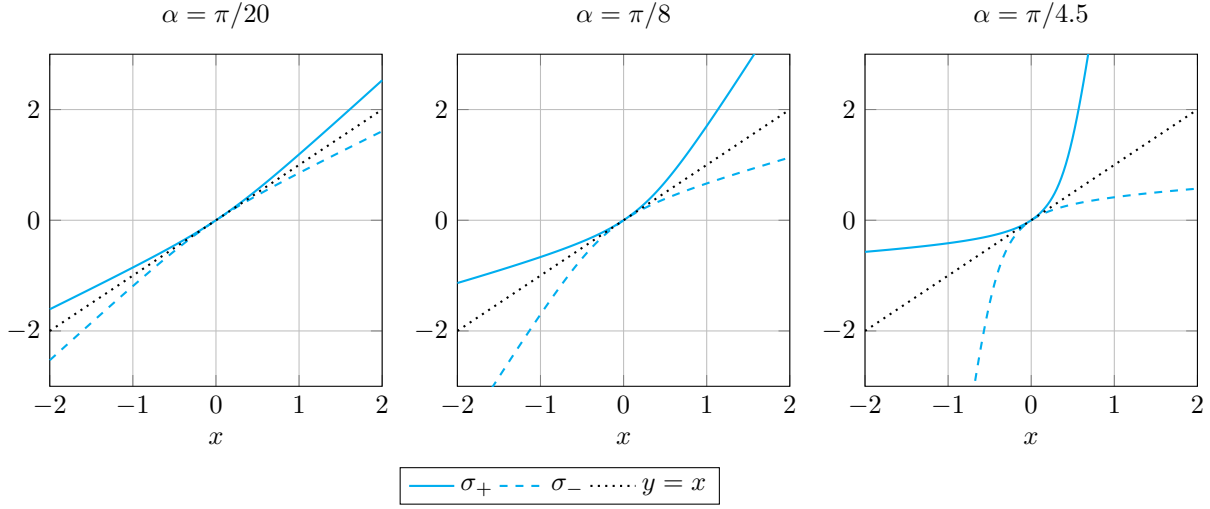
\begin{figure}
    \centering
    \begin{tikzpicture}
    \begin{groupplot}[
        group style={
            group size=3 by 1,
            horizontal sep=1cm,
        },
        width=0.35\linewidth,
        height=6cm,
        xlabel={$x$},
        xmin=-2, xmax=2,
        ymin=-3, ymax=3,
        grid=major,
        legend style={at={(-1.4,-0.3)}, anchor=west, legend columns=3},
        axis equal=false,
    ]

    % Loop through different alpha values for different subplots
    \pgfplotsforeachungrouped \alphanopi/\plotcolor in {20/cyan,8/cyan,4.5/cyan} {
        \pgfmathsetmacro{\alphaval}{pi/\alphanopi}
        \pgfmathsetmacro{\a}{1/(sin(\alphaval r)*sin(\alphaval r)) - 1/(cos(\alphaval r)*cos(\alphaval r))}
        \pgfmathsetmacro{\b}{1/(sin(\alphaval r)*sin(\alphaval r)) + 1/(cos(\alphaval r)*cos(\alphaval r))}
        \edef\tmp{
        \noexpand\nextgroupplot[
            title={$\alpha= \pi / \alphanopi$},
        ] % Name each plot for coordinate reference
        \noexpand\addplot[\plotcolor, thick, domain=-2:2, samples=150, label={$\sigma_+$}] {
        (\b*x)/\a - sqrt(2)/(\a*sin(\alphaval r)) + 
        (1/\a)*sqrt(((2*x)/(sin(\alphaval r)*cos(\alphaval r)) - sqrt(2)/cos(\alphaval r))^2 + 2*\a)
        };
        \noexpand\addplot[\plotcolor, dashed, thick, domain=-2:2, samples=150, label={$\sigma_-$}] {
            (\b*x)/\a + sqrt(2)/(\a*sin(\alphaval r)) - 
            (1/\a)*sqrt(((2*x)/(sin(\alphaval r)*cos(\alphaval r)) + sqrt(2)/cos(\alphaval r))^2 + 2*\a)
        };
        \noexpand\addplot[black, dotted, thick, domain=-2:2] {x};
        }
        \tmp
        }
        \legend{$\sigma_+$, $\sigma_-$, $y=x$}
    \end{groupplot}
\end{tikzpicture}
    \label{fig:activation-functions}
    \caption{Activation functions $\sigma_{\pm}$ for different values of $\alpha$.}
\end{figure}
\begin{equation}
    \rho^{(l)}(\vec{x}^{(l)}) = \sigma_- \left( \vec{\Psi}_l^T (\vec{x}^{(l)} - \vec{\beta}_l) \right)\quad \text{and} \quad
    \varphi^{(l)}(\vec{z}^{(l-1)}) = \vec{\Phi}_l \sigma_+ \left(\vec{z}^{(l-1)}\right) + \vec{\beta}_l ,
\end{equation}
where $\vec{\Psi}_l \in \mathbb{R}^{n_{l-1} \times n_l}$ and $\vec{\Phi}_l \in \mathbb{R}^{n_l \times n_{l-1}}$ are the weight matrices for the encoder and decoder layers, respectively. The pair $(\sigma_-, \sigma_+)$ consists of smooth activation functions that are inverses of one another, and $\vec{\beta}_l$ are the bias vectors.

For each layer to satisfy the idempotent property, a crucial condition must be met by the weight matrices:
\begin{equation}
    \vec{\Psi}_l^T \vec{\Phi}_l = \vec{I}_{\mathbb{R}^{n_{l-1}}} .
\end{equation}
This defines $\vec{\Phi}_l$ and $\vec{\Psi}_l$ as biorthogonal matrices. We enforce this biorthogonality constraint during training via Riemannian optimization on the biorthogonal manifold using the \texttt{geoopt} library \citep{kochurovGeooptRiemannianOptimization2020,friedl_riemannian_2025}.

\section{Neuromodulation details}
\label{sec:appendix-neuromodulation}
Neuromodulation in biological neural systems refers to the process by which
neuromodulators modulate neuronal activity and synaptic properties,
enabling adaptive responses. In artificial neural networks, this concept translates into dynamically adjusting network parameters based on contextual information, allowing the same network architecture to exhibit different behaviors depending on the input context.

In our neuromodulated constrained autoencoder, we implement context-dependent activation functions where the activation parameters $\vec{\alpha}^{(l)}$ and bias vectors $\vec{\beta}^{(l)}$ for a pair of layers $l$ are modulated by a context vector $\vec{c}$.

\paragraph{Modulated activation functions}

The encoder and decoder layers in our method use context-dependent activation functions:
\begin{align}
    \rho^{(l)}(\vec{x}^{(l)}) &= \sigma_- \left( \vec{\Psi}_l^T (\vec{x}^{(l)} - \vec{\beta}^{(l)}); \vec{\alpha}^{(l)} \right), \\
    \varphi^{(l)}(\vec{z}^{(l-1)}) &= \vec{\Phi}_l \sigma_+ \left(\vec{z}^{(l-1)}; \vec{\alpha}^{(l)}\right) + \vec{\beta}^{(l)},
\end{align}
where the parameter $\alpha$ becomes variable and is specific to each input dimension, such that it writes now $\vec{\alpha}^{(l)} \in \mathbb{R}^{n_{l-1}}$.
In addition, the values $\vec{\alpha}^{(l)}$ are constrained to lie within a specific range $[\frac{\pi}{30}, \frac{\pi}{6}]$ which is more restrictive than the initial setting with the range $]0,\frac{\pi}{4}[$~.
The reason behind this can be shown in \cref{fig:act-func-issue} where we can see that for $\alpha$ values close to $\pi/4$, we have very steep slopes which can cause issues with training.
Alternatively, the biases $\vec{\beta}^{(l)}$ are left unconstrained.

\begin{figure}
        \centering
        \begin{tikzpicture}
    \begin{groupplot}[
        group style={
            group size=1 by 1,
            horizontal sep=1cm,
        },
        width=1.0\linewidth,
        height=6cm,
        xlabel={$x$},
        xmin=-2, xmax=2,
        ymin=-10, ymax=10,
        grid=major,
        axis equal=false,
    ]
    
    % Loop through different alpha values for different subplots
    \pgfplotsforeachungrouped \alphanopi/\plotcolor in {4.05/cyan} {
        \pgfmathsetmacro{\alphaval}{pi/\alphanopi}
        \pgfmathsetmacro{\a}{1/(sin(\alphaval r)*sin(\alphaval r)) - 1/(cos(\alphaval r)*cos(\alphaval r))}
        \pgfmathsetmacro{\b}{1/(sin(\alphaval r)*sin(\alphaval r)) + 1/(cos(\alphaval r)*cos(\alphaval r))}
        \edef\tmp{
        \noexpand\nextgroupplot[
            title={$\alpha= \pi / \alphanopi$},
        ] % Name each plot for coordinate reference
        \noexpand\addplot[\plotcolor, very thick, domain=-2:1, samples=150, label={$\sigma_+$}] {
        (\b*x)/\a - sqrt(2)/(\a*sin(\alphaval r)) + 
        (1/\a)*sqrt(((2*x)/(sin(\alphaval r)*cos(\alphaval r)) - sqrt(2)/cos(\alphaval r))^2 + 2*\a)
        };
        \noexpand\addplot[\plotcolor, dashed, very thick, domain=-1:2, samples=150, label={$\sigma_-$}] {
            (\b*x)/\a + sqrt(2)/(\a*sin(\alphaval r)) - 
            (1/\a)*sqrt(((2*x)/(sin(\alphaval r)*cos(\alphaval r)) + sqrt(2)/cos(\alphaval r))^2 + 2*\a)
        };
        \noexpand\addplot[black, dotted, thick, domain=-2:2] {x};
        }
        \tmp
        }
        \legend{$\sigma_+$, $\sigma_-$, $y=x$}
    \end{groupplot}
\end{tikzpicture}
        \caption{Activation functions $\sigma_{\pm}$ for $\alpha$ near $\frac{\pi}{4}$.}
        \label{fig:act-func-issue}
\end{figure}
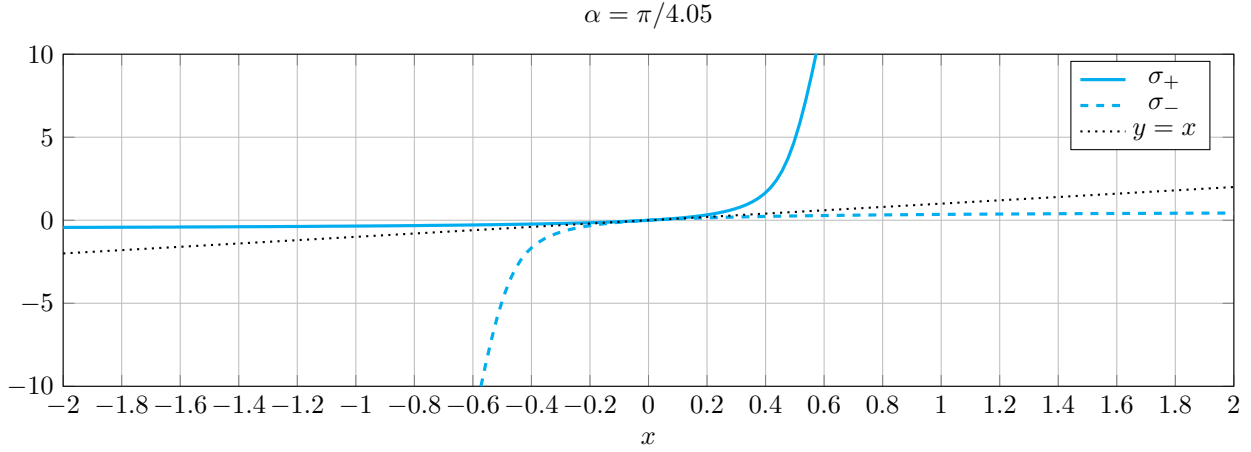

\paragraph{Two-stage neuromodulation process}

Stage 1: The context vector $\vec{c} \in \mathbb{R}^{d_c}$ is processed through a fully-connected neural network to generate a neuromodulation signal $\vec{s} = f_{\text{nmd}}(\vec{c}; \vec{\theta})$.

Stage 2: For each pair of layers $l$, the neuromodulation signal is transformed into layer-pair-specific activation parameters:
\begin{equation}
    \vec{\alpha}^{(l)} = (\alpha_{\max}-\alpha_{\min})\cdot\textsc{sigmoid}(\vec{W}_{l, \alpha}^T \vec{s}) + \alpha_{\min},
\end{equation}
and
\begin{equation}
    \vec{\beta}^{(l)} = \vec{W}^T_{l, \beta} \vec{s} + \vec{b}_l,
\end{equation}
where $\vec{W}_{l,\alpha}$, $\vec{W}_{l,\beta}$, and $\vec{b}_l$ are learnable layer-specific parameters.
 
\section{Experimental details}
\label{sec:exp-details}
\subsection{Training objective}
\label{sec:appendix-loss}
All architectures are trained with the same loss, combining position reconstruction and a Jacobian-weighted velocity term:
\begin{equation*}
\mathcal{L}_{ae} = \frac{1}{N} \sum_{i=1}^N \left\|\vec{x}_i - P(\vec{x}_i)\right\|^2 + \left\|\dot{\vec{x}}_i - \nabla_{\vec{x}}P(\vec{x}_i)\cdot \dot{\vec{x}}_i\right\|^2,
\end{equation*}
where $N$ is the batch size and $P = \varphi \circ \rho$ is the reconstruction map.

% -------------------------------------------------------
\subsection{16-DoF pendulum}
\label{sec:appendix-pendulum}

The first four pendulum links are modeled in \emph{MuJoCo}~\citep{todorov2012mujoco} as capsules with radius $0.05\,\si{\metre}$ and mass $1.0\,\si{\kilogram}$ each.
The remaining 12 DoF are derived from the first four joint angles through nonlinear coupling functions.
The context vector $\vec{c} = (l_1, l_2, l_3, l_4)$ contains the four link lengths.

\subsubsection{Coupling functions}
The context-dependent coupling functions used in the main experiment are shown in \cref{fig:coupling-tables}.

\begin{table}[t]
    \begin{center}
        \caption{Coupling functions.}
        \label{fig:coupling-tables}
        \begin{subtable}{.49\linewidth}
            \caption{Original coupling from \cite{friedl_riemannian_2025}}
            \label{fig:coupling-tables-a}
            \centerline{\begin{tabular}{c|c}
                DoF & $f(q_1, q_2, q_3, q_4)$ \\
                \midrule
                $q_5$ & $q_3 - \cos(q_2)$ \\
                $q_6$ & $q_1 + 0.1\sin(q_2)$ \\
                $q_7$ & $q_4\cos(q_2)$ \\
                $q_8$ & $q_1 + q_3^2$ \\
                $q_9$ & $1.5\sin(q_2)$ \\
                $q_{10}$ & $-q_4q_1$ \\
                $q_{11}$ & $\sin(q_1)$ \\
                $q_{12}$ & $0.4q_3q_4$ \\
                $q_{13}$ & $-0.9q_1-q_2+q_3-2q_4^2$ \\
                $q_{14}$ & $-3\sin(q_3)$ \\
                $q_{15}$ & $-2q_3^2$ \\
                $q_{16}$ & $-0.9q_1^2$ \\
            \end{tabular}}
        \end{subtable}
        \begin{subtable}{.49\linewidth}
            \caption{Context-dependent coupling}
            \label{fig:coupling-tables-b}
            \centerline{\begin{tabular}{c|c}
                DoF & $f(q_1, q_2, q_3, q_4, {\color{blue}l_1, l_2, l_3, l_4})$ \\
                \midrule
                $q_5$ & $q_3 - \cos({\color{blue}2l_2}q_2)$ \\
                $q_6$ & $q_1 + {\color{blue}2l_1}0.1\sin({\color{blue}2l_2}q_2)$ \\
                $q_7$ & $q_4\cos({\color{blue}2l_4}q_2)$ \\
                $q_8$ & $q_1 + q_3^{2\cdot{\color{blue}2l_3}}$ \\
                $q_9$ & ${\color{blue}2l_2}1.5\sin(q_2)$ \\
                $q_{10}$ & $-{\color{blue}(l_4+l_1)}q_4q_1$ \\
                $q_{11}$ & $\sin({\color{blue}2l_1}q_1)$ \\
                $q_{12}$ & ${\color{blue}2l_3}0.4q_3q_4$ \\
                $q_{13}$ & $-{\color{blue}2l_1}0.9q_1-q_2+q_3-2q_4^{2\cdot{\color{blue}2l_4}}$ \\
                $q_{14}$ & $-{\color{blue}2l_3}3\sin(q_3)$ \\
                $q_{15}$ & $-2q_3^{2\cdot{\color{blue}2l_3}}$ \\
                $q_{16}$ & $-0.9q_1^{2\cdot{\color{blue}2l_1}}$ \\
            \end{tabular}}
        \end{subtable}
    \end{center}
\end{table}

\subsubsection{Data collection}
\label{sec:data-pendulum}
For each trajectory, link lengths $l_1,\ldots,l_4$ are independently and uniformly sampled from $[0.35, 0.65]\,\si{\metre}$.
\paragraph{Training set.} 100 trajectories are simulated in MuJoCo for $3\,\si{\second}$ at a timestep of $1\,\si{\ms}$.
For each configuration, the initial angles of the first four joints are drawn uniformly from $[0\si{\degree}, 30\si{\degree}]$.
\paragraph{Test set.} 256 configurations are obtained by taking all combinations from the grid $\{0.35, 0.45, 0.55, 0.65\}^4\,\si{\metre}$.
All joints are initialized at $15\si{\degree}$, providing a systematic evaluation across the full context range.

\subsubsection{Hyperparameters}
\label{sec:hyper-pendulum}
The layer sizes $n_l$ refer to the biorthogonal layer pairs within the encoder $\rho^{(l)}:\mathbb{R}^{n_{l}}\to\mathbb{R}^{n_{l-1}}$ and the decoder $\varphi^{(l)}:\mathbb{R}^{n_{l-1}}\to\mathbb{R}^{n_{l}}$.
Hyperparameters for the constrained architectures are given in \cref{tab:pendulum-params-constrained} and for the unconstrained baselines in \cref{tab:pendulum-params-unconstrained}.

\begin{table}[ht]
\centering
\caption{Hyperparameters for the 16-DoF pendulum --- constrained architectures (\cae, \contextcae, \ncae).}
\label{tab:pendulum-params-constrained}
\centering\medskip
\begin{tabular}{lccc}
\toprule
\textbf{Hyperparameter} & \textbf{\cae} & \textbf{\contextcae} & \textbf{\ncae} \\
\midrule
\textit{Main architecture} & & & \\
Layer sizes ($n_l$) & $[8, 16, 16, 16]$ & \multicolumn{2}{c}{$[8, 12, 14, 16]$} \\
Latent dimension ($d$) & \multicolumn{3}{c}{4} \\
Activation parameter $\alpha$ & $\pi/8$ & $\pi/10$ & N/A \\
\midrule
\textit{Contextual mechanism} & & & \\
Mechanism & None & Concatenation & Neuromodulation \\
MLP topology & \multicolumn{2}{c}{N/A} & $[4,4,4]$ \\
MLP activation & \multicolumn{2}{c}{N/A} & SiLU \\
\midrule
\textit{Optimization} & & & \\
Optimizer & \multicolumn{3}{c}{Riemannian Adam (\texttt{geoopt})} \\
Learning rate & \multicolumn{3}{c}{$5\cdot10^{-2}$} \\
Weight decay & \multicolumn{3}{c}{$10^{-5}$} \\
Epochs & \multicolumn{3}{c}{10000} \\
Batch size & $2048$ & \multicolumn{2}{c}{$4096$} \\
Scheduler & \multicolumn{3}{c}{ReduceLROnPlateau (patience 250, factor 0.9)} \\
\bottomrule
\end{tabular}
\end{table}

\begin{table}[ht]
\centering
\caption{Hyperparameters for the 16-DoF pendulum --- unconstrained baselines (AE, \contextAE{}, FiLM+AE, SoftNcAE). $\dagger$~\contextAE{} input dimension is $20 = 16 + 4$ due to context concatenation.}
\label{tab:pendulum-params-unconstrained}
\centering\medskip
\begin{tabular}{lcccc}
\toprule
\textbf{Hyperparameter} & \textbf{AE} & \textbf{\contextAE{}} & \textbf{FiLM+AE} & \textbf{SoftNcAE} \\
\midrule
\textit{Main architecture} & & & & \\
Encoder layers & $[16,16,16,8,4]$ & $[20,16,16,8,4]^\dagger$ & $[16,16,16,8,4]$ & $[16,14,12,8,4]$ \\
Decoder layers & $[4,8,16,16,16]$ & $[4,8,16,16,16]$ & $[4,8,16,16,16]$ & $[4,8,12,14,16]$ \\
Latent dimension ($d$) & \multicolumn{4}{c}{4} \\
Activation function & ReLU & ReLU & SiLU & $\sigma_\pm$ \\
\midrule
\textit{Contextual mechanism} & & & & \\
Mechanism & None & Concatenation & FiLM & Neuromodulation \\
Soft penalty coefficient & \multicolumn{3}{c}{N/A} & $1.0$ \\
MLP topology & \multicolumn{2}{c}{N/A} & per-layer linear & $[4,4]$ \\
MLP activation & \multicolumn{3}{c}{N/A} & SiLU \\
\midrule
\textit{Optimization} & & & & \\
Optimizer & \multicolumn{4}{c}{Adam} \\
Learning rate & $10^{-2}$ & \multicolumn{2}{c}{$10^{-3}$} & $5\cdot10^{-2}$ \\
Weight decay & \multicolumn{3}{c}{$10^{-6}$} & $10^{-5}$ \\
Epochs & \multicolumn{4}{c}{10000} \\
Batch size & \multicolumn{3}{c}{$2048$} & $4096$ \\
Scheduler & \multicolumn{4}{c}{ReduceLROnPlateau (patience 250, factor 0.9)} \\
\bottomrule
\end{tabular}
\end{table}

\subsection{Lorenz96}
\label{sec:appendix-lorenz}

The Lorenz96 system~\citep{75462} consists of $N=36$ variables governed by
\begin{equation*}
    \dot{x}_k = (x_{k+1} - x_{k-2})\,x_{k-1} - x_k + F, \qquad k = 1,\ldots,N,
\end{equation*}
with periodic boundary conditions $x_k = x_{k+N}$.
The context is the scalar forcing constant $\vec{c} = F$.

\subsubsection{Data collection}
\label{sec:data-lorenz}
Trajectories are integrated with a 4th-order Runge-Kutta method at $\Delta t = 0.01$.
The initial condition $x_k(0) = F + \sin\!\left(8.05\cdot\tfrac{2\pi k}{N}\right)$ places the state near the attractor associated with each value of $F$.
A transient period is discarded before recording; each trajectory then consists of 500 time points.
\paragraph{Training set.} 18 trajectories with $F$ sampled uniformly from $[3.133, 3.193]$.
\paragraph{Test set.} 10 values of $F$ linearly spaced across $[3.133, 3.193]$.

\subsubsection{Forcing regimes}
Hovmöller diagrams in~\cref{fig:app-F-regimes} illustrate the two dynamical regimes on either side of the bifurcation.
The system transitions from a traveling wave with spatial frequency $8$ at $F=3.133$ to frequency $7$ at $F=3.193$.
\begin{figure*}[!h]
\centering
    \begin{subfigure}[b]{0.49\linewidth}
        \includegraphics[width=\linewidth]{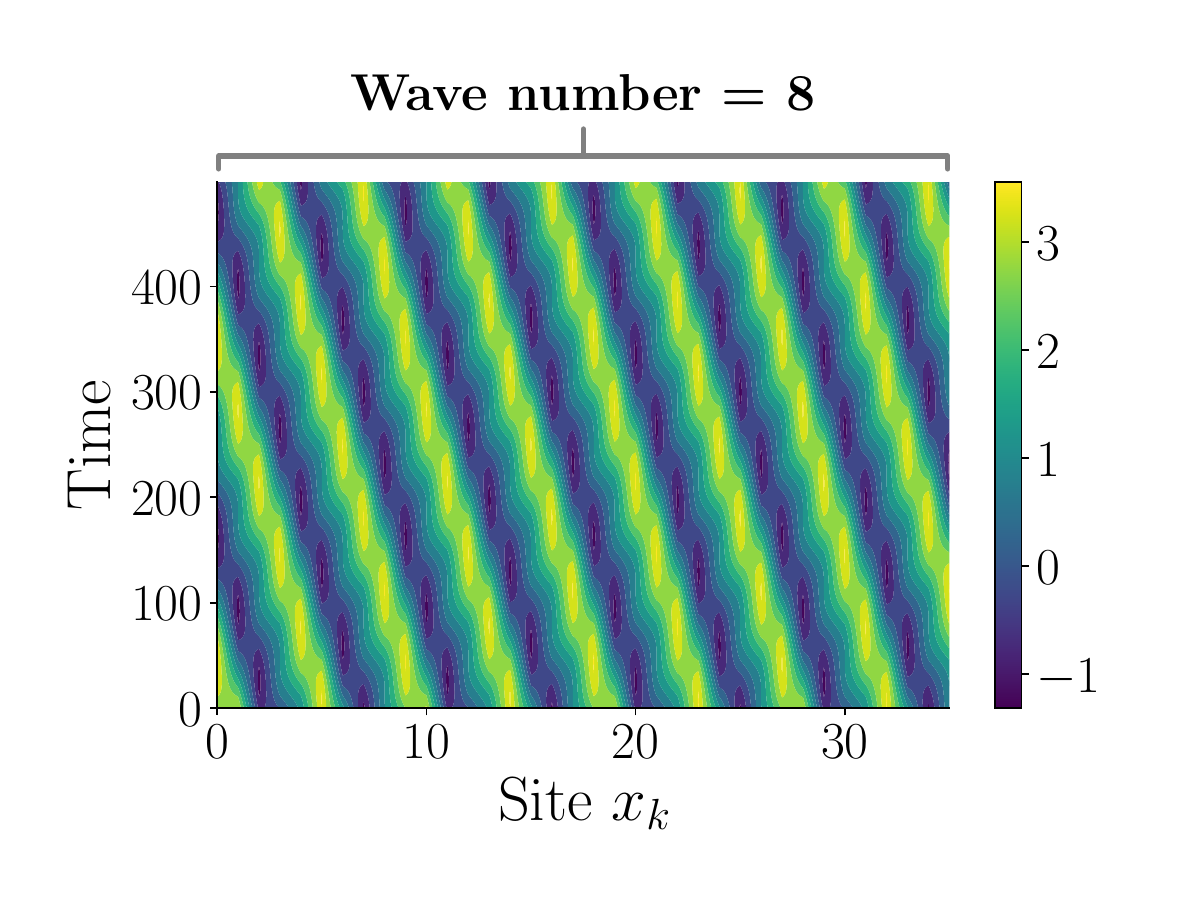}
        \caption{$F=3.133$}
    \end{subfigure}
    \begin{subfigure}[b]{0.49\linewidth}
        \includegraphics[width=\linewidth]{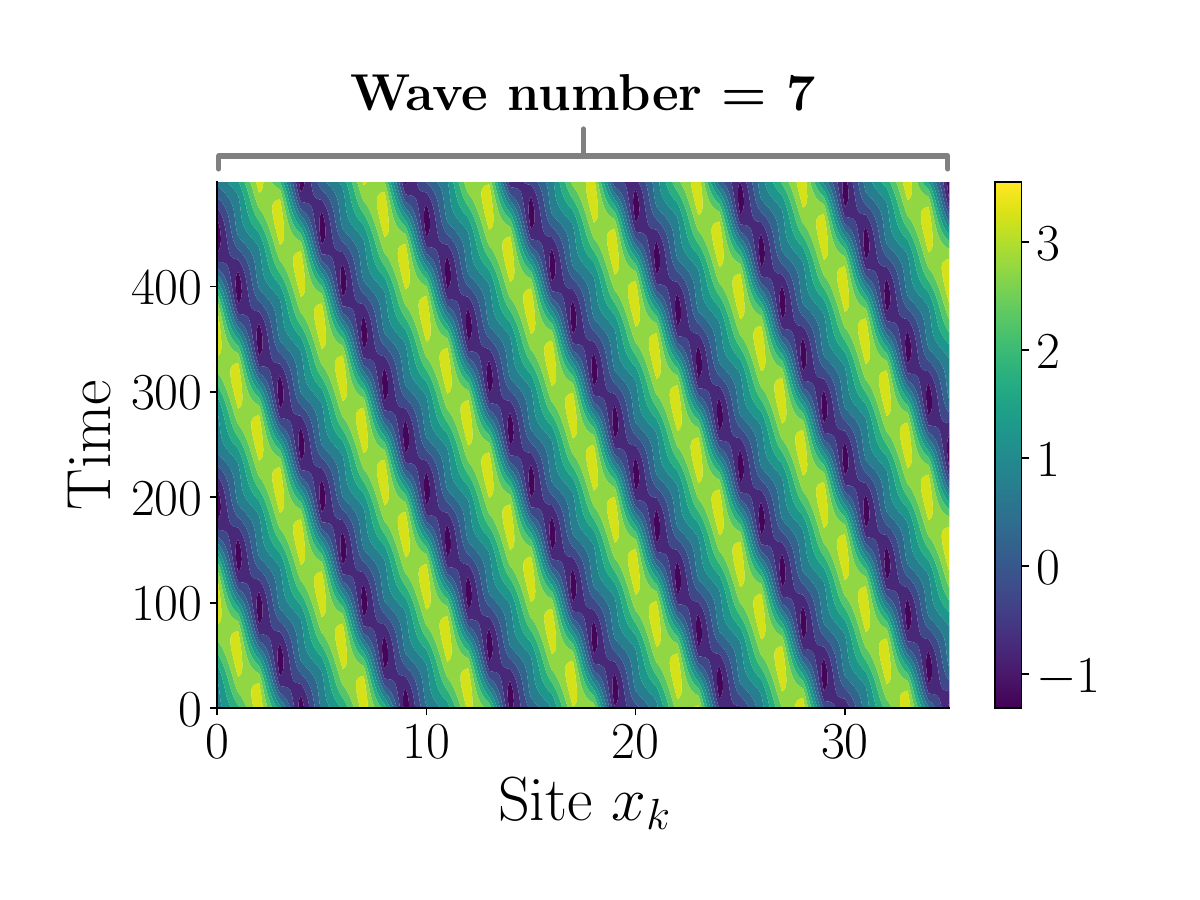}
        \caption{$F=3.193$}
    \end{subfigure}
    \caption{
        Hovmöller diagrams of the Lorenz96 system for values of $F$ before and after the bifurcation.
        }
    \label{fig:app-F-regimes}
\end{figure*}

\subsubsection{Hyperparameters}
\label{sec:hyper-lorenz}
Hyperparameters for the constrained architectures are given in \cref{tab:lorenz-params-constrained} and for the unconstrained baselines in \cref{tab:lorenz-params-unconstrained}.

\begin{table}[!h]
\centering
\caption{Hyperparameters for the Lorenz96 experiment --- constrained architectures (\cae, \contextcae, \ncae).}
\label{tab:lorenz-params-constrained}
\centering\medskip
\begin{tabular}{lccc}
\toprule
\textbf{Hyperparameter} & \textbf{\cae} & \textbf{\contextcae} & \textbf{\ncae} \\
\midrule
\textit{Main architecture} & & & \\
Layer sizes ($n_l$) & $[21, 36]$ & $[20, 36]$ & $[18, 36]$ \\
Latent dimension ($d$) & \multicolumn{3}{c}{2} \\
Activation parameter $\alpha$ & $\pi/8$ & $\pi/10$ & N/A \\
\midrule
\textit{Contextual mechanism} & & & \\
Mechanism & None & Concatenation & Neuromodulation \\
MLP topology & \multicolumn{2}{c}{N/A} & $[1,2,2,2]$ \\
MLP activation & \multicolumn{2}{c}{N/A} & SiLU \\
\midrule
\textit{Optimization} & & & \\
Optimizer & \multicolumn{3}{c}{Riemannian Adam (\texttt{geoopt})} \\
Learning rate & \multicolumn{3}{c}{$5\cdot10^{-2}$} \\
Weight decay & \multicolumn{3}{c}{$10^{-6}$} \\
Epochs & \multicolumn{3}{c}{10000} \\
Batch size & $512$ & $256$ & $128$ \\
Scheduler & \multicolumn{3}{c}{ReduceLROnPlateau (patience 250, factor 0.9)} \\
\bottomrule
\end{tabular}
\end{table}

\begin{table}[!h]
\centering
\caption{Hyperparameters for the Lorenz96 experiment --- unconstrained baselines (AE, \contextAE{}, FiLM+AE, SoftNcAE). $\dagger$~\contextAE{} input dimension is $37 = 36 + 1$ due to context concatenation.}
\label{tab:lorenz-params-unconstrained}
\centering\medskip
\begin{tabular}{lcccc}
\toprule
\textbf{Hyperparameter} & \textbf{AE} & \textbf{\contextAE{}} & \textbf{FiLM+AE} & \textbf{SoftNcAE} \\
\midrule
\textit{Main architecture} & & & & \\
Encoder layers & $[36,24,12,2]$ & $[37,20,2]^\dagger$ & $[36,24,12,2]$ & $[36,18,2]$ \\
Decoder layers & $[2,12,24,36]$ & $[2,12,24,36]$ & $[2,12,24,36]$ & $[2,18,36]$ \\
Latent dimension ($d$) & \multicolumn{4}{c}{2} \\
Activation function & ReLU & ReLU & SiLU & $\sigma_\pm$ \\
\midrule
\textit{Contextual mechanism} & & & & \\
Mechanism & None & Concatenation & FiLM & Neuromodulation \\
Soft penalty coefficient & \multicolumn{3}{c}{N/A} & $1.0$ \\
MLP topology & \multicolumn{2}{c}{N/A} & per-layer linear & $[1,2,2,2]$ \\
MLP activation & \multicolumn{3}{c}{N/A} & SiLU \\
\midrule
\textit{Optimization} & & & & \\
Optimizer & \multicolumn{4}{c}{Adam} \\
Learning rate & $10^{-2}$ & \multicolumn{2}{c}{$10^{-3}$} & $5\cdot10^{-2}$ \\
Weight decay & \multicolumn{4}{c}{$10^{-6}$} \\
Epochs & \multicolumn{4}{c}{10000} \\
Batch size & $256$ & \multicolumn{3}{c}{$128$}\\
Scheduler & \multicolumn{4}{c}{ReduceLROnPlateau (patience 250, factor 0.9)} \\
\bottomrule
\end{tabular}
\end{table}

\subsection{Compute ressources}
\label{sec:compute-ressource}
All training experiments were performed on an academic HPC cluster, using a single NVIDIA Tesla A100 GPU (40 GB VRAM) per run on a node with AMD EPYC Milan CPUs and up to 512 GB of system RAM. The total compute footprint of the project, including preliminary experiments not reported in this paper, remains modest by modern deep learning standards.

% -------------------------------------------------------
\clearpage
\section{Robustness results}
\label{sec:appendix-robustness}

\subsection{Out-of-distribution contexts}
\label{ssec:OOD-robust}
\begin{table}[!h]
    \centering
    \caption{
        \textbf{Out-of-distribution context performance on Lorenz96} (mean $\pm$ std RMSE).
        OOD distance is expressed as a percentage beyond the training boundary.
        The projection remains geometrically valid by construction (\cref{prop:hausdorff});
        Only reconstruction fidelity degrades with distance.
    }\medskip
    \label{tab:lorenz_ood}
    \resizebox{\linewidth}{!}{\begin{tabular}{lcccccc}
        \toprule
        & \multicolumn{2}{c}{\textbf{OOD +10\%}} & \multicolumn{2}{c}{\textbf{OOD +20\%}} & \multicolumn{2}{c}{\textbf{OOD +50\%}} \\
        \cmidrule(lr){2-3} \cmidrule(lr){4-5} \cmidrule(lr){6-7}
        \textbf{Model} & \textbf{Pos. RMSE} $\downarrow$ & \textbf{Vel. RMSE} $\downarrow$ & \textbf{Pos. RMSE} $\downarrow$ & \textbf{Vel. RMSE} $\downarrow$ & \textbf{Pos. RMSE} $\downarrow$ & \textbf{Vel. RMSE} $\downarrow$ \\
        \midrule
        \cae        & $0.467_{\pm 0.250}$ & $0.649_{\pm 0.099}$ & $0.468_{\pm 0.250}$ & $0.650_{\pm 0.097}$ & $0.472_{\pm 0.250}$ & $0.659_{\pm 0.096}$ \\
        \contextcae & $0.315_{\pm 0.251}$ & $0.646_{\pm 0.197}$ & $0.318_{\pm 0.252}$ & $0.654_{\pm 0.197}$ & $0.326_{\pm 0.253}$ & $0.689_{\pm 0.195}$ \\
        AE          & $0.176_{\pm 0.091}$ & $16.65_{\pm 6.327}$ & $0.157_{\pm 0.087}$ & $12.38_{\pm 5.653}$ & $\mathbf{0.177_{\pm 0.092}}$ & $18.09_{\pm 7.681}$ \\
        \contextAE{}   & $0.237_{\pm 0.051}$ & $10.81_{\pm 3.046}$ & $0.234_{\pm 0.049}$ & $10.32_{\pm 2.658}$ & $0.251_{\pm 0.048}$ & $10.72_{\pm 2.887}$ \\
        FiLM+AE     & $0.928_{\pm 0.420}$ & $5.989_{\pm 2.940}$ & $2.594_{\pm 2.212}$ & $14.87_{\pm 12.06}$ & $43.66_{\pm 52.39}$ & $273.9_{\pm 473.7}$ \\
        SoftNcAE    & $0.734_{\pm 0.224}$ & $2.221_{\pm 0.544}$ & $0.736_{\pm 0.223}$ & $2.223_{\pm 0.543}$ & $0.747_{\pm 0.217}$ & $2.235_{\pm 0.540}$ \\
        \midrule
        \ncae{} (ours) & $\mathbf{0.083_{\pm 0.069}}$ & $\mathbf{0.330_{\pm 0.204}}$ & $\mathbf{0.105_{\pm 0.098}}$ & $\mathbf{0.360_{\pm 0.252}}$ & $0.240_{\pm 0.350}$ & $\mathbf{0.545_{\pm 0.638}}$ \\
        \bottomrule
    \end{tabular}}
\end{table}

\begin{table}[!h]
    \centering
    \caption{
        \textbf{Out-of-distribution context performance on the pendulum} (mean $\pm$ std RMSE).
        OOD distance is expressed as a percentage beyond the training boundary.
        The projection remains geometrically valid by construction (\cref{prop:hausdorff});
        Only reconstruction fidelity degrades with distance.
    }\medskip
    \label{tab:pendulum_ood}
    \resizebox{\linewidth}{!}{\begin{tabular}{lcccccc}
        \toprule
        & \multicolumn{2}{c}{\textbf{OOD +10\%}} & \multicolumn{2}{c}{\textbf{OOD +20\%}} & \multicolumn{2}{c}{\textbf{OOD +50\%}} \\
        \cmidrule(lr){2-3} \cmidrule(lr){4-5} \cmidrule(lr){6-7}
        \textbf{Model} & \textbf{Pos. RMSE} $\downarrow$ & \textbf{Vel. RMSE} $\downarrow$ & \textbf{Pos. RMSE} $\downarrow$ & \textbf{Vel. RMSE} $\downarrow$ & \textbf{Pos. RMSE} $\downarrow$ & \textbf{Vel. RMSE} $\downarrow$ \\
        \midrule
        \cae        & $0.079_{\pm 0.002}$ & $0.499_{\pm 0.021}$ & $0.099_{\pm 0.003}$ & $0.634_{\pm 0.041}$ & $0.117_{\pm 0.003}$ & $0.794_{\pm 0.055}$ \\
        \contextcae & $0.087_{\pm 0.001}$ & $0.495_{\pm 0.006}$ & $0.108_{\pm 0.002}$ & $0.625_{\pm 0.013}$ & $0.131_{\pm 0.005}$ & $0.803_{\pm 0.029}$ \\
        AE          & $0.082_{\pm 0.004}$ & $2.919_{\pm 3.836}$ & $0.103_{\pm 0.004}$ & $3.708_{\pm 4.885}$ & $0.121_{\pm 0.006}$ & $4.252_{\pm 5.724}$ \\
        \contextAE{}   & $0.084_{\pm 0.004}$ & $1.226_{\pm 0.811}$ & $0.105_{\pm 0.005}$ & $1.458_{\pm 0.947}$ & $0.124_{\pm 0.007}$ & $1.587_{\pm 0.843}$ \\
        FiLM+AE     & $\mathbf{0.021_{\pm 0.029}}$ & $0.518_{\pm 0.935}$ & $\mathbf{0.033_{\pm 0.053}}$ & $0.656_{\pm 1.136}$ & $0.079_{\pm 0.156}$ & $0.777_{\pm 0.958}$ \\
        SoftNcAE    & $0.045_{\pm 0.008}$ & $0.260_{\pm 0.036}$ & $0.057_{\pm 0.011}$ & $0.353_{\pm 0.050}$ & $0.086_{\pm 0.021}$ & $0.545_{\pm 0.069}$ \\
        \midrule
        \ncae{} (ours) & $0.027_{\pm 0.005}$ & $\mathbf{0.152_{\pm 0.016}}$ & $0.034_{\pm 0.007}$ & $\mathbf{0.218_{\pm 0.025}}$ & $\mathbf{0.060_{\pm 0.013}}$ & $\mathbf{0.420_{\pm 0.041}}$ \\
        \bottomrule
    \end{tabular}}
\end{table}

\subsection{Noise robustness: models trained on noisy data and evaluated on clean data}
\label{ssec:noise-robust-train}
% ---- LORENZ96 NOISE TABLES ----

\begin{table}[!h]
    \caption{
        \textbf{Lorenz96 --- state noise only} (mean$_{\pm\text{std}}$ RMSE).
        Gaussian noise $\mathcal{N}(0, \sigma^2)$ is added to states at training time.
    }
    \label{tab:lorenz-noise-state-test}
    \medskip
    \centering
    \resizebox{\linewidth}{!}{\begin{tabular}{l cc cc cc}
        \toprule
        & \multicolumn{2}{c}{$\sigma=0.10$} & \multicolumn{2}{c}{$\sigma=0.25$} & \multicolumn{2}{c}{$\sigma=0.50$} \\
        \cmidrule(lr){2-3} \cmidrule(lr){4-5} \cmidrule(lr){6-7}
        \textbf{Model} & \textbf{Pos.} & \textbf{Vel.} & \textbf{Pos.} & \textbf{Vel.} & \textbf{Pos.} & \textbf{Vel.} \\
        \midrule
        \cae        & $0.336_{\pm 0.154}$ & $0.613_{\pm 0.075}$ & $0.534_{\pm 0.251}$ & $0.908_{\pm 0.228}$ & $0.526_{\pm 0.288}$ & $1.179_{\pm 0.288}$ \\
        \contextcae & $0.453_{\pm 0.252}$ & $0.745_{\pm 0.128}$ & $0.499_{\pm 0.267}$ & $0.951_{\pm 0.253}$ & $0.409_{\pm 0.238}$ & $1.082_{\pm 0.191}$ \\
        AE          & $0.165_{\pm 0.061}$ & $21.077_{\pm 7.938}$ & $0.179_{\pm 0.045}$ & $17.901_{\pm 8.541}$ & $0.260_{\pm 0.048}$ & $16.748_{\pm 4.014}$ \\
        \contextAE{}   & $0.260_{\pm 0.046}$ & $10.125_{\pm 2.014}$ & $0.300_{\pm 0.053}$ & $9.014_{\pm 1.829}$ & $0.346_{\pm 0.054}$ & $8.221_{\pm 1.912}$ \\
        FiLM+AE     & $0.327_{\pm 0.072}$ & $2.792_{\pm 2.610}$ & $0.335_{\pm 0.060}$ & $2.633_{\pm 2.187}$ & $0.379_{\pm 0.054}$ & $2.674_{\pm 1.850}$ \\
        SoftNcAE    & $0.711_{\pm 0.215}$ & $2.299_{\pm 0.399}$ & $0.775_{\pm 0.231}$ & $2.294_{\pm 0.551}$ & $0.783_{\pm 0.165}$ & $2.428_{\pm 0.206}$ \\
        \midrule
        \ncae{} (ours) & $\mathbf{0.052_{\pm 0.030}}$ & $\mathbf{0.231_{\pm 0.139}}$ & $\mathbf{0.118_{\pm 0.140}}$ & $\mathbf{0.431_{\pm 0.413}}$ & $\mathbf{0.117_{\pm 0.062}}$ & $\mathbf{0.380_{\pm 0.120}}$ \\
        \bottomrule
    \end{tabular}}
\end{table}

\begin{table}[!h]
    \caption{
        \textbf{Lorenz96 --- context noise only} (mean$_{\pm\text{std}}$ RMSE).
        Gaussian noise $\mathcal{N}(0, \sigma^2)$ added to context at training time.
        Context-free architectures (\cae, AE) are omitted as they are insensitive to context noise.
    }
    \label{tab:lorenz-noise-ctx-test}
    \medskip
    \centering
    \resizebox{\linewidth}{!}{\begin{tabular}{l cc cc cc}
        \toprule
        & \multicolumn{2}{c}{$\sigma=0.10$} & \multicolumn{2}{c}{$\sigma=0.25$} & \multicolumn{2}{c}{$\sigma=0.50$} \\
        \cmidrule(lr){2-3} \cmidrule(lr){4-5} \cmidrule(lr){6-7}
        \textbf{Model} & \textbf{Pos.} & \textbf{Vel.} & \textbf{Pos.} & \textbf{Vel.} & \textbf{Pos.} & \textbf{Vel.} \\
        \midrule
        \contextcae & $0.136_{\pm 0.063}$ & $0.463_{\pm 0.064}$ & $\mathbf{0.166_{\pm 0.126}}$ & $0.498_{\pm 0.115}$ & $0.230_{\pm 0.200}$ & $0.603_{\pm 0.267}$ \\
        \contextAE{}   & $0.240_{\pm 0.039}$ & $11.363_{\pm 2.372}$ & $0.239_{\pm 0.038}$ & $11.714_{\pm 2.510}$ & $0.240_{\pm 0.039}$ & $11.663_{\pm 2.381}$ \\
        FiLM+AE     & $0.329_{\pm 0.037}$ & $2.897_{\pm 2.065}$ & $0.338_{\pm 0.032}$ & $3.207_{\pm 2.004}$ & $0.395_{\pm 0.033}$ & $3.893_{\pm 2.980}$ \\
        SoftNcAE    & $0.844_{\pm 0.111}$ & $2.265_{\pm 0.299}$ & $0.986_{\pm 0.113}$ & $2.575_{\pm 0.379}$ & $1.003_{\pm 0.121}$ & $2.609_{\pm 0.448}$ \\
        \midrule
        \ncae{} (ours) & $\mathbf{0.083_{\pm 0.013}}$ & $\mathbf{0.319_{\pm 0.031}}$ & $0.170_{\pm 0.035}$ & $\mathbf{0.400_{\pm 0.041}}$ & $\mathbf{0.194_{\pm 0.093}}$ & $\mathbf{0.465_{\pm 0.094}}$ \\
        \bottomrule
    \end{tabular}}
\end{table}

\begin{table}[!h]
    \caption{
        \textbf{Lorenz96 --- state and context noise} (mean$_{\pm\text{std}}$ RMSE).
        Gaussian noise $\mathcal{N}(0, \sigma^2)$ added to both states and context at training time.
        Context-free architectures (\cae, AE) are omitted as they are insensitive to context noise.
    }
    \label{tab:lorenz-noise-both-test}
    \medskip
    \centering
    \resizebox{\linewidth}{!}{\begin{tabular}{l cc cc cc}
        \toprule
        & \multicolumn{2}{c}{$\sigma=0.10$} & \multicolumn{2}{c}{$\sigma=0.25$} & \multicolumn{2}{c}{$\sigma=0.50$} \\
        \cmidrule(lr){2-3} \cmidrule(lr){4-5} \cmidrule(lr){6-7}
        \textbf{Model} & \textbf{Pos.} & \textbf{Vel.} & \textbf{Pos.} & \textbf{Vel.} & \textbf{Pos.} & \textbf{Vel.} \\
        \midrule
        \contextcae & $0.222_{\pm 0.087}$ & $0.516_{\pm 0.074}$ & $0.403_{\pm 0.228}$ & $0.842_{\pm 0.293}$ & $0.449_{\pm 0.235}$ & $0.982_{\pm 0.180}$ \\
        \contextAE{}   & $0.263_{\pm 0.048}$ & $9.966_{\pm 1.854}$ & $0.298_{\pm 0.056}$ & $8.990_{\pm 1.793}$ & $0.346_{\pm 0.057}$ & $8.289_{\pm 2.000}$ \\
        FiLM+AE     & $0.334_{\pm 0.036}$ & $2.363_{\pm 0.853}$ & $0.368_{\pm 0.034}$ & $3.417_{\pm 2.000}$ & $0.436_{\pm 0.038}$ & $4.144_{\pm 3.268}$ \\
        SoftNcAE    & $0.876_{\pm 0.143}$ & $2.342_{\pm 0.414}$ & $0.972_{\pm 0.104}$ & $2.605_{\pm 0.307}$ & $1.000_{\pm 0.078}$ & $2.684_{\pm 0.296}$ \\
        \midrule
        \ncae{} (ours) & $\mathbf{0.088_{\pm 0.021}}$ & $\mathbf{0.362_{\pm 0.043}}$ & $\mathbf{0.193_{\pm 0.047}}$ & $\mathbf{0.606_{\pm 0.064}}$ & $\mathbf{0.209_{\pm 0.051}}$ & $\mathbf{0.766_{\pm 0.051}}$ \\
        \bottomrule
    \end{tabular}}
\end{table}

% ---- PENDULUM NOISE TABLES (data pending) ----

\begin{table}[!h]
    \caption{
        \textbf{Pendulum --- state noise only} (mean$_{\pm\text{std}}$ RMSE).
        Gaussian noise $\mathcal{N}(0, \sigma^2)$ added to states at training time.
    }
    \label{tab:pend-noise-state-test}
    \medskip
    \centering
    \resizebox{\linewidth}{!}{\begin{tabular}{l cc cc cc}
        \toprule
        & \multicolumn{2}{c}{$\sigma=0.10$} & \multicolumn{2}{c}{$\sigma=0.25$} & \multicolumn{2}{c}{$\sigma=0.50$} \\
        \cmidrule(lr){2-3} \cmidrule(lr){4-5} \cmidrule(lr){6-7}
        \textbf{Model} & \textbf{Pos.} & \textbf{Vel.} & \textbf{Pos.} & \textbf{Vel.} & \textbf{Pos.} & \textbf{Vel.} \\
        \midrule
        \cae        & $0.047_{\pm 0.001}$ & $0.210_{\pm 0.002}$ & $0.045_{\pm 0.001}$ & $0.203_{\pm 0.004}$ & $0.045_{\pm 0.001}$ & $0.201_{\pm 0.003}$ \\
        \contextcae & $0.054_{\pm 0.001}$ & $0.224_{\pm 0.002}$ & $0.054_{\pm 0.001}$ & $0.224_{\pm 0.002}$ & $0.055_{\pm 0.002}$ & $0.226_{\pm 0.004}$ \\
        AE          & $0.050_{\pm 0.001}$ & $0.635_{\pm 0.223}$ & $0.050_{\pm 0.001}$ & $2.201_{\pm 4.181}$ & $0.051_{\pm 0.002}$ & $5.692_{\pm 5.969}$ \\
        \contextAE{}   & $0.051_{\pm 0.001}$ & $0.586_{\pm 0.303}$ & $0.050_{\pm 0.001}$ & $0.432_{\pm 0.045}$ & $0.053_{\pm 0.001}$ & $1.060_{\pm 1.018}$ \\
        FiLM+AE     & $\mathbf{0.010_{\pm 0.012}}$ & $0.220_{\pm 0.366}$ & $\mathbf{0.008_{\pm 0.007}}$ & $0.083_{\pm 0.032}$ & $\mathbf{0.013_{\pm 0.008}}$ & $0.230_{\pm 0.381}$ \\
        SoftNcAE    & $0.038_{\pm 0.027}$ & $0.142_{\pm 0.067}$ & $0.043_{\pm 0.026}$ & $0.159_{\pm 0.063}$ & $0.201_{\pm 0.440}$ & $0.292_{\pm 0.281}$ \\
        \midrule
        \ncae{} (ours) & $0.012_{\pm 0.001}$ & $\mathbf{0.058_{\pm 0.005}}$ & $0.013_{\pm 0.002}$ & $\mathbf{0.070_{\pm 0.006}}$ & $0.018_{\pm 0.002}$ & $\mathbf{0.091_{\pm 0.006}}$ \\
        \bottomrule
    \end{tabular}}
\end{table}

\begin{table}[!h]
    \caption{
        \textbf{Pendulum --- context noise only} (mean$_{\pm\text{std}}$ RMSE).
        Gaussian noise $\mathcal{N}(0, \sigma^2)$ added to context at training time.
        Context-free architectures (\cae, AE) are omitted as they are insensitive to context noise.
    }
    \label{tab:pend-noise-ctx-test}
    \medskip
    \centering
    \resizebox{\linewidth}{!}{\begin{tabular}{l cc cc cc}
        \toprule
        & \multicolumn{2}{c}{$\sigma=0.10$} & \multicolumn{2}{c}{$\sigma=0.25$} & \multicolumn{2}{c}{$\sigma=0.50$} \\
        \cmidrule(lr){2-3} \cmidrule(lr){4-5} \cmidrule(lr){6-7}
        \textbf{Model} & \textbf{Pos.} & \textbf{Vel.} & \textbf{Pos.} & \textbf{Vel.} & \textbf{Pos.} & \textbf{Vel.} \\
        \midrule
        \contextcae & $0.054_{\pm 0.001}$ & $0.225_{\pm 0.002}$ & $0.054_{\pm 0.001}$ & $0.225_{\pm 0.002}$ & $0.052_{\pm 0.001}$ & $0.222_{\pm 0.002}$ \\
        \contextAE{}   & $0.051_{\pm 0.001}$ & $0.492_{\pm 0.188}$ & $0.050_{\pm 0.001}$ & $0.859_{\pm 0.940}$ & $0.050_{\pm 0.001}$ & $0.582_{\pm 0.403}$ \\
        FiLM+AE     & $\mathbf{0.007_{\pm 0.008}}$ & $0.288_{\pm 0.706}$ & $\mathbf{0.006_{\pm 0.001}}$ & $\mathbf{0.055_{\pm 0.005}}$ & $\mathbf{0.013_{\pm 0.001}}$ & $0.167_{\pm 0.234}$ \\
        SoftNcAE    & $0.037_{\pm 0.016}$ & $0.146_{\pm 0.052}$ & $0.032_{\pm 0.014}$ & $0.131_{\pm 0.046}$ & $0.047_{\pm 0.046}$ & $0.174_{\pm 0.103}$\\
        \midrule
        \ncae{} (ours) & $0.013_{\pm 0.002}$ & $\mathbf{0.058_{\pm 0.009}}$ & $0.015_{\pm 0.003}$ & $0.069_{\pm 0.006}$ & $0.017_{\pm 0.001}$ & $\mathbf{0.088_{\pm 0.004}}$  \\
        \bottomrule
    \end{tabular}}
\end{table}

\begin{table}[!h]
    \caption{
        \textbf{Pendulum --- state and context noise} (mean$_{\pm\text{std}}$ RMSE).
        Gaussian noise $\mathcal{N}(0, \sigma^2)$ added to both states and context at training time.
        Context-free architectures (\cae, AE) are omitted as they are insensitive to context noise.
    }
    \label{tab:pend-noise-both-test}
    \medskip
    \centering
    \resizebox{\linewidth}{!}{\begin{tabular}{l cc cc cc}
        \toprule
        & \multicolumn{2}{c}{$\sigma=0.10$} & \multicolumn{2}{c}{$\sigma=0.25$} & \multicolumn{2}{c}{$\sigma=0.50$} \\
        \cmidrule(lr){2-3} \cmidrule(lr){4-5} \cmidrule(lr){6-7}
        \textbf{Model} & \textbf{Pos.} & \textbf{Vel.} & \textbf{Pos.} & \textbf{Vel.} & \textbf{Pos.} & \textbf{Vel.} \\
        \midrule
        \contextcae & $0.054_{\pm 0.001}$ & $0.223_{\pm 0.002}$ & $0.053_{\pm 0.001}$ & $0.221_{\pm 0.002}$ & $0.051_{\pm 0.001}$ & $0.220_{\pm 0.002}$ \\
        \contextAE{}   & $0.050_{\pm 0.001}$ & $0.627_{\pm 0.381}$ & $0.051_{\pm 0.001}$ & $0.522_{\pm 0.223}$ & $0.052_{\pm 0.001}$ & $0.587_{\pm 0.166}$ \\
        FiLM+AE     & $\mathbf{0.007_{\pm 0.007}}$ & $0.064_{\pm 0.022}$ & $\mathbf{0.008_{\pm 0.0004}}$ & $0.077_{\pm 0.004}$ & $\mathbf{0.017_{\pm 0.002}}$ & $0.365_{\pm 0.659}$ \\
        SoftNcAE    & $0.045_{\pm 0.043}$ & $0.159_{\pm 0.104}$ & $0.042_{\pm 0.010}$ & $0.160_{\pm 0.038}$ & $0.045_{\pm 0.015}$ & $0.175_{\pm 0.033}$ \\
        \midrule
        \ncae{} (ours) & $0.013_{\pm 0.002}$ & $\mathbf{0.064_{\pm 0.006}}$ & $0.014_{\pm 0.001}$ & $\mathbf{0.073_{\pm 0.005}}$ & $0.021_{\pm 0.002}$ & $\mathbf{0.107_{\pm 0.005}}$ \\
        \bottomrule
    \end{tabular}}
\end{table}
\clearpage
\subsection{Noise robustness: models trained on clean data and evaluated on noisy data}
\label{ssec:noise-robust-test}
% auto-generated by noise_comparison_table_from_metrics — do not edit by hand
\begin{table}[!h]
    \caption{
        \textbf{Lorenz96 --- state and context noise} (mean$_{\pm\text{std}}$ RMSE).
        Gaussian noise $\mathcal{N}(0, \sigma^2)$ added to both states and context at test time; models trained on clean data.
        Context-free architectures (\cae, AE) are omitted as they are insensitive to context noise.
    }
    \label{tab:lorenz_noise_both}
    \medskip
    \centering
    \resizebox{\linewidth}{!}{\begin{tabular}{l cc cc cc}
        \toprule
        & \multicolumn{2}{c}{$\sigma=0.10$} & \multicolumn{2}{c}{$\sigma=0.25$} & \multicolumn{2}{c}{$\sigma=0.50$} \\
        \cmidrule(lr){2-3} \cmidrule(lr){4-5} \cmidrule(lr){6-7}
        \textbf{Model} & \textbf{Pos.} & \textbf{Vel.} & \textbf{Pos.} & \textbf{Vel.} & \textbf{Pos.} & \textbf{Vel.} \\
        \midrule
        \contextcae & $0.680_{\pm 0.297}$ & $1.928_{\pm 0.592}$ & $1.413_{\pm 0.579}$ & $4.015_{\pm 1.261}$ & $2.776_{\pm 1.232}$ & $7.099_{\pm 2.254}$ \\
        \contextAE{}   & $0.378_{\pm 0.093}$ & $11.611_{\pm 2.839}$ & $0.725_{\pm 0.202}$ & $11.963_{\pm 2.820}$ & $1.272_{\pm 0.357}$ & $12.988_{\pm 2.848}$ \\
        FiLM+AE     & $0.462_{\pm 0.037}$ & $3.406_{\pm 2.773}$ & $1.348_{\pm 0.733}$ & $7.264_{\pm 4.455}$ & $8.677_{\pm 10.330}$ & $34.120_{\pm 40.071}$ \\
        SoftNcAE    & $0.820_{\pm 0.192}$ & $2.319_{\pm 0.507}$ & $0.962_{\pm 0.149}$ & $2.589_{\pm 0.410}$ & $1.210_{\pm 0.152}$ & $3.063_{\pm 0.351}$ \\
        \midrule
        \ncae{} (ours) & $\mathbf{0.236_{\pm 0.024}}$ & $\mathbf{0.657_{\pm 0.089}}$ & $\mathbf{0.536_{\pm 0.047}}$ & $\mathbf{1.373_{\pm 0.108}}$ & $\mathbf{1.001_{\pm 0.099}}$ & $\mathbf{2.469_{\pm 0.235}}$ \\
        \bottomrule
    \end{tabular}}
\end{table}

% auto-generated by noise_comparison_table_from_metrics — do not edit by hand
\begin{table}[!h]
    \caption{
        \textbf{Lorenz96 --- context noise only} (mean$_{\pm\text{std}}$ RMSE).
        Gaussian noise $\mathcal{N}(0, \sigma^2)$ added to context at test time; models trained on clean data.
        Context-free architectures (\cae, AE) are omitted as they are insensitive to context noise.
    }
    \label{tab:lorenz_noise_context}
    \medskip
    \centering
    \resizebox{\linewidth}{!}{\begin{tabular}{l cc cc cc}
        \toprule
        & \multicolumn{2}{c}{$\sigma=0.10$} & \multicolumn{2}{c}{$\sigma=0.25$} & \multicolumn{2}{c}{$\sigma=0.50$} \\
        \cmidrule(lr){2-3} \cmidrule(lr){4-5} \cmidrule(lr){6-7}
        \textbf{Model} & \textbf{Pos.} & \textbf{Vel.} & \textbf{Pos.} & \textbf{Vel.} & \textbf{Pos.} & \textbf{Vel.} \\
        \midrule
        \contextcae & $0.319_{\pm 0.254}$ & $0.644_{\pm 0.204}$ & $0.319_{\pm 0.253}$ & $\mathbf{0.645_{\pm 0.204}}$ & $0.320_{\pm 0.253}$ & $\mathbf{0.650_{\pm 0.203}}$ \\
        \contextAE{}   & $0.237_{\pm 0.046}$ & $11.142_{\pm 2.794}$ & $\mathbf{0.237_{\pm 0.046}}$ & $11.108_{\pm 2.817}$ & $\mathbf{0.239_{\pm 0.046}}$ & $11.084_{\pm 2.761}$ \\
        FiLM+AE     & $0.370_{\pm 0.066}$ & $3.004_{\pm 2.735}$ & $0.848_{\pm 0.443}$ & $4.747_{\pm 2.597}$ & $7.147_{\pm 9.760}$ & $31.344_{\pm 48.653}$ \\
        SoftNcAE    & $0.810_{\pm 0.195}$ & $2.306_{\pm 0.512}$ & $0.920_{\pm 0.157}$ & $2.548_{\pm 0.422}$ & $1.085_{\pm 0.137}$ & $2.909_{\pm 0.353}$ \\
        \midrule
        \ncae{} (ours) & $\mathbf{0.148_{\pm 0.039}}$ & $\mathbf{0.503_{\pm 0.158}}$ & $0.343_{\pm 0.038}$ & $0.993_{\pm 0.145}$ & $0.637_{\pm 0.066}$ & $1.712_{\pm 0.178}$ \\
        \bottomrule
    \end{tabular}}
\end{table}

% auto-generated by noise_comparison_table_from_metrics — do not edit by hand
\begin{table}[!h]
    \caption{
        \textbf{Lorenz96 --- state noise only} (mean$_{\pm\text{std}}$ RMSE).
        Gaussian noise $\mathcal{N}(0, \sigma^2)$ added to states at test time; models trained on clean data.
    }
    \label{tab:lorenz_noise_state}
    \medskip
    \centering
    \resizebox{\linewidth}{!}{\begin{tabular}{l cc cc cc}
        \toprule
        & \multicolumn{2}{c}{$\sigma=0.10$} & \multicolumn{2}{c}{$\sigma=0.25$} & \multicolumn{2}{c}{$\sigma=0.50$} \\
        \cmidrule(lr){2-3} \cmidrule(lr){4-5} \cmidrule(lr){6-7}
        \textbf{Model} & \textbf{Pos.} & \textbf{Vel.} & \textbf{Pos.} & \textbf{Vel.} & \textbf{Pos.} & \textbf{Vel.} \\
        \midrule
        \cae        & $0.725_{\pm 0.277}$ & $1.822_{\pm 0.366}$ & $1.406_{\pm 0.456}$ & $3.914_{\pm 0.796}$ & $2.690_{\pm 0.882}$ & $6.983_{\pm 1.416}$ \\
        \contextcae & $0.680_{\pm 0.297}$ & $1.928_{\pm 0.592}$ & $1.413_{\pm 0.579}$ & $4.014_{\pm 1.261}$ & $2.776_{\pm 1.232}$ & $7.098_{\pm 2.254}$ \\
        AE          & $0.255_{\pm 0.066}$ & $19.152_{\pm 6.605}$ & $0.468_{\pm 0.098}$ & $18.846_{\pm 6.660}$ & $0.836_{\pm 0.151}$ & $18.608_{\pm 5.445}$ \\
        \contextAE{}   & $0.378_{\pm 0.093}$ & $11.595_{\pm 2.828}$ & $0.724_{\pm 0.202}$ & $11.984_{\pm 2.810}$ & $1.271_{\pm 0.356}$ & $13.071_{\pm 2.810}$ \\
        FiLM+AE     & $0.416_{\pm 0.043}$ & $3.281_{\pm 2.815}$ & $0.801_{\pm 0.292}$ & $4.857_{\pm 2.840}$ & $2.015_{\pm 1.711}$ & $9.491_{\pm 7.808}$ \\
        SoftNcAE    & $0.757_{\pm 0.222}$ & $2.247_{\pm 0.538}$ & $0.801_{\pm 0.203}$ & $2.294_{\pm 0.517}$ & $0.929_{\pm 0.163}$ & $2.443_{\pm 0.460}$ \\
        \midrule
        \ncae{} (ours) & $\mathbf{0.181_{\pm 0.032}}$ & $\mathbf{0.515_{\pm 0.111}}$ & $\mathbf{0.362_{\pm 0.053}}$ & $\mathbf{0.916_{\pm 0.125}}$ & $\mathbf{0.705_{\pm 0.121}}$ & $\mathbf{1.680_{\pm 0.283}}$ \\
        \bottomrule
    \end{tabular}}
\end{table}

% auto-generated by noise_comparison_table_from_metrics — do not edit by hand
\begin{table}[!h]
    \caption{
        \textbf{Pendulum --- state and context noise} (mean$_{\pm\text{std}}$ RMSE).
        Gaussian noise $\mathcal{N}(0, \sigma^2)$ added to both states and context at test time; models trained on clean data.
        Context-free architectures (\cae, AE) are omitted as they are insensitive to context noise.
    }
    \label{tab:pendulum_noise_both}
    \medskip
    \centering
    \resizebox{\linewidth}{!}{\begin{tabular}{l cc cc cc}
        \toprule
        & \multicolumn{2}{c}{$\sigma=0.10$} & \multicolumn{2}{c}{$\sigma=0.25$} & \multicolumn{2}{c}{$\sigma=0.50$} \\
        \cmidrule(lr){2-3} \cmidrule(lr){4-5} \cmidrule(lr){6-7}
        \textbf{Model} & \textbf{Pos.} & \textbf{Vel.} & \textbf{Pos.} & \textbf{Vel.} & \textbf{Pos.} & \textbf{Vel.} \\
        \midrule
        \contextcae & $0.063_{\pm 0.002}$ & $0.237_{\pm 0.002}$ & $0.098_{\pm 0.006}$ & $0.291_{\pm 0.010}$ & $0.192_{\pm 0.024}$ & $0.496_{\pm 0.086}$ \\
        \contextAE{}   & $0.055_{\pm 0.002}$ & $0.908_{\pm 0.809}$ & $0.070_{\pm 0.002}$ & $0.975_{\pm 0.790}$ & $\mathbf{0.108_{\pm 0.007}}$ & $1.137_{\pm 0.853}$ \\
        FiLM+AE     & $0.027_{\pm 0.010}$ & $0.376_{\pm 0.755}$ & $0.063_{\pm 0.019}$ & $0.461_{\pm 0.775}$ & $0.126_{\pm 0.040}$ & $0.639_{\pm 0.883}$ \\
        SoftNcAE    & $0.035_{\pm 0.004}$ & $0.129_{\pm 0.016}$ & $\mathbf{0.063_{\pm 0.004}}$ & $\mathbf{0.190_{\pm 0.011}}$ & $0.128_{\pm 0.010}$ & $\mathbf{0.357_{\pm 0.032}}$ \\
        \midrule
        \ncae{} (ours) & $\mathbf{0.027_{\pm 0.002}}$ & $\mathbf{0.091_{\pm 0.007}}$ & $0.068_{\pm 0.008}$ & $0.203_{\pm 0.022}$ & $0.180_{\pm 0.036}$ & $0.520_{\pm 0.137}$ \\
        \bottomrule
    \end{tabular}}
\end{table}

% auto-generated by noise_comparison_table_from_metrics — do not edit by hand
\begin{table}[!h]
    \caption{
        \textbf{Pendulum --- context noise only} (mean$_{\pm\text{std}}$ RMSE).
        Gaussian noise $\mathcal{N}(0, \sigma^2)$ added to context at test time; models trained on clean data.
        Context-free architectures (\cae, AE) are omitted as they are insensitive to context noise.
    }
    \label{tab:pendulum_noise_context}
    \medskip
    \centering
    \resizebox{\linewidth}{!}{\begin{tabular}{l cc cc cc}
        \toprule
        & \multicolumn{2}{c}{$\sigma=0.10$} & \multicolumn{2}{c}{$\sigma=0.25$} & \multicolumn{2}{c}{$\sigma=0.50$} \\
        \cmidrule(lr){2-3} \cmidrule(lr){4-5} \cmidrule(lr){6-7}
        \textbf{Model} & \textbf{Pos.} & \textbf{Vel.} & \textbf{Pos.} & \textbf{Vel.} & \textbf{Pos.} & \textbf{Vel.} \\
        \midrule
        \contextcae & $0.054_{\pm 0.001}$ & $0.226_{\pm 0.001}$ & $0.055_{\pm 0.001}$ & $0.227_{\pm 0.001}$ & $0.055_{\pm 0.001}$ & $0.229_{\pm 0.001}$ \\
        \contextAE{}   & $0.051_{\pm 0.002}$ & $0.881_{\pm 0.802}$ & $0.051_{\pm 0.002}$ & $0.869_{\pm 0.774}$ & $0.051_{\pm 0.002}$ & $0.882_{\pm 0.811}$ \\
        FiLM+AE     & $\mathbf{0.010_{\pm 0.011}}$ & $0.351_{\pm 0.750}$ & $0.018_{\pm 0.020}$ & $0.365_{\pm 0.745}$ & $0.034_{\pm 0.041}$ & $0.399_{\pm 0.740}$ \\
        SoftNcAE    & $0.026_{\pm 0.005}$ & $0.115_{\pm 0.019}$ & $0.028_{\pm 0.005}$ & $0.120_{\pm 0.018}$ & $0.035_{\pm 0.005}$ & $0.138_{\pm 0.016}$ \\
        \midrule
        \ncae{} (ours) & $0.014_{\pm 0.003}$ & $\mathbf{0.061_{\pm 0.007}}$ & $\mathbf{0.017_{\pm 0.003}}$ & $\mathbf{0.071_{\pm 0.006}}$ & $\mathbf{0.025_{\pm 0.003}}$ & $\mathbf{0.100_{\pm 0.011}}$ \\
        \bottomrule
    \end{tabular}}
\end{table}

% auto-generated by noise_comparison_table_from_metrics — do not edit by hand
\begin{table}[!h]
    \caption{
        \textbf{Pendulum --- state noise only} (mean$_{\pm\text{std}}$ RMSE).
        Gaussian noise $\mathcal{N}(0, \sigma^2)$ added to states at test time; models trained on clean data.
    }
    \label{tab:pendulum_noise_state}
    \medskip
    \centering
    \resizebox{\linewidth}{!}{\begin{tabular}{l cc cc cc}
        \toprule
        & \multicolumn{2}{c}{$\sigma=0.10$} & \multicolumn{2}{c}{$\sigma=0.25$} & \multicolumn{2}{c}{$\sigma=0.50$} \\
        \cmidrule(lr){2-3} \cmidrule(lr){4-5} \cmidrule(lr){6-7}
        \textbf{Model} & \textbf{Pos.} & \textbf{Vel.} & \textbf{Pos.} & \textbf{Vel.} & \textbf{Pos.} & \textbf{Vel.} \\
        \midrule
        \cae        & $0.063_{\pm 0.004}$ & $0.245_{\pm 0.011}$ & $0.139_{\pm 0.041}$ & $0.423_{\pm 0.129}$ & $0.435_{\pm 0.248}$ & $1.150_{\pm 0.660}$ \\
        \contextcae & $0.063_{\pm 0.002}$ & $0.237_{\pm 0.002}$ & $0.097_{\pm 0.006}$ & $0.290_{\pm 0.009}$ & $0.190_{\pm 0.023}$ & $0.491_{\pm 0.082}$ \\
        AE          & $0.053_{\pm 0.001}$ & $2.040_{\pm 2.974}$ & $0.066_{\pm 0.002}$ & $2.220_{\pm 3.065}$ & $\mathbf{0.097_{\pm 0.003}}$ & $2.464_{\pm 3.355}$ \\
        \contextAE{}   & $0.055_{\pm 0.002}$ & $0.905_{\pm 0.805}$ & $0.070_{\pm 0.002}$ & $0.988_{\pm 0.824}$ & $0.108_{\pm 0.007}$ & $1.136_{\pm 0.845}$ \\
        FiLM+AE     & $\mathbf{0.026_{\pm 0.009}}$ & $0.374_{\pm 0.755}$ & $\mathbf{0.060_{\pm 0.011}}$ & $0.455_{\pm 0.786}$ & $0.116_{\pm 0.016}$ & $0.619_{\pm 0.886}$ \\
        SoftNcAE    & $0.034_{\pm 0.004}$ & $0.128_{\pm 0.016}$ & $0.062_{\pm 0.004}$ & $\mathbf{0.186_{\pm 0.012}}$ & $0.124_{\pm 0.010}$ & $\mathbf{0.344_{\pm 0.034}}$ \\
        \midrule
        \ncae{} (ours) & $0.027_{\pm 0.002}$ & $\mathbf{0.090_{\pm 0.007}}$ & $0.067_{\pm 0.008}$ & $0.197_{\pm 0.021}$ & $0.175_{\pm 0.035}$ & $0.498_{\pm 0.123}$ \\
        \bottomrule
    \end{tabular}}
\end{table}

\clearpage
\section{Idempotency tables}
\begin{table}[!h]
    \centering
    \caption{
        \textbf{Idempotency error} $\|P^k(\vec{x}) - P^{k-1}(\vec{x})\|$ (mean $\pm$ std) over successive projection steps $k$, reported on the Lorenz96 system.
        Constrained architectures (\cae, \ncae) maintain error $\approx 0$ by construction; $\infty$ indicates numerical blow-up.
    }\medskip
    \label{tab:idempotency}
    \resizebox{\linewidth}{!}{\begin{tabular}{lcccc}
        \toprule
        \textbf{Method} & \textbf{Step 1} & \textbf{Step 5} & \textbf{Step 10} & \textbf{Step 15} \\
        \midrule
        \cae            & $1.46 \times 10^{-5}_{\pm 4.99 \times 10^{-6}}$ & $1.38 \times 10^{-5}_{\pm 4.87 \times 10^{-6}}$ & $1.33 \times 10^{-5}_{\pm 4.85 \times 10^{-6}}$ & $1.28 \times 10^{-5}_{\pm 4.99 \times 10^{-6}}$ \\
        \contextcae     & $1.68 \times 10^{0}_{\pm 1.36 \times 10^{0}}$ & $8.45 \times 10^{0}_{\pm 2.19 \times 10^{1}}$ & $1.71 \times 10^{2}_{\pm 5.84 \times 10^{2}}$ & $4.22 \times 10^{3}_{\pm 1.58 \times 10^{4}}$ \\
        AE              & $4.99 \times 10^{-1}_{\pm 2.82 \times 10^{-1}}$ & $5.26 \times 10^{-1}_{\pm 3.02 \times 10^{-1}}$ & $2.58 \times 10^{0}_{\pm 4.22 \times 10^{0}}$ & $6.73 \times 10^{1}_{\pm 1.83 \times 10^{2}}$ \\
        \contextAE{}       & $8.10 \times 10^{-1}_{\pm 3.79 \times 10^{-1}}$ & $1.24 \times 10^{0}_{\pm 9.71 \times 10^{-1}}$ & $2.25 \times 10^{0}_{\pm 1.66 \times 10^{0}}$ & $8.65 \times 10^{0}_{\pm 1.45 \times 10^{1}}$ \\
        FiLM+AE         & $1.13 \times 10^{0}_{\pm 2.91 \times 10^{-1}}$ & $1.29 \times 10^{6}_{\pm 3.86 \times 10^{6}}$ & $\infty$ & $\infty$ \\
        SoftNcAE        & $9.98 \times 10^{-1}_{\pm 6.87 \times 10^{-1}}$ & $2.27 \times 10^{0}_{\pm 2.72 \times 10^{0}}$ & $1.01 \times 10^{1}_{\pm 2.62 \times 10^{1}}$ & $8.92 \times 10^{1}_{\pm 3.82 \times 10^{2}}$ \\
        \midrule
        \ncae{} (ours)  & $5.28 \times 10^{-6}_{\pm 2.82 \times 10^{-6}}$ & $5.01 \times 10^{-6}_{\pm 2.77 \times 10^{-6}}$ & $4.85 \times 10^{-6}_{\pm 2.75 \times 10^{-6}}$ & $4.73 \times 10^{-6}_{\pm 2.76 \times 10^{-6}}$ \\
        \bottomrule
    \end{tabular}}
\end{table}

\begin{table}[!h]
    \centering
    \caption{
        \textbf{Idempotency error} $\|P^k(\vec{x}) - P^{k-1}(\vec{x})\|$ (mean $\pm$ std) over successive projection steps $k$, reported on the pendulum system.
        Constrained architectures (\cae, \ncae) maintain error $\approx 0$ by construction.
    }\medskip
    \label{tab:pendulum_all_idempotency}
    \resizebox{\linewidth}{!}{\begin{tabular}{lcccc}
        \toprule
        \textbf{Method} & \textbf{Step 1} & \textbf{Step 5} & \textbf{Step 10} & \textbf{Step 15} \\
        \midrule
        \cae            & $1.26 \times 10^{-5}_{\pm 2.53 \times 10^{-6}}$ & $1.25 \times 10^{-5}_{\pm 2.52 \times 10^{-6}}$ & $1.24 \times 10^{-5}_{\pm 2.52 \times 10^{-6}}$ & $1.24 \times 10^{-5}_{\pm 2.51 \times 10^{-6}}$ \\
        \contextcae     & $2.72 \times 10^{-1}_{\pm 3.33 \times 10^{-2}}$ & $3.09 \times 10^{-1}_{\pm 5.96 \times 10^{-2}}$ & $3.38 \times 10^{1}_{\pm 1.00 \times 10^{2}}$ & $1.20 \times 10^{5}_{\pm 3.60 \times 10^{5}}$ \\
        AE              & $9.77 \times 10^{-2}_{\pm 1.65 \times 10^{-2}}$ & $8.80 \times 10^{-2}_{\pm 1.52 \times 10^{-2}}$ & $8.18 \times 10^{-2}_{\pm 1.57 \times 10^{-2}}$ & $7.85 \times 10^{-2}_{\pm 1.64 \times 10^{-2}}$ \\
        \contextAE{}       & $1.59 \times 10^{-1}_{\pm 1.04 \times 10^{-2}}$ & $1.46 \times 10^{-1}_{\pm 1.27 \times 10^{-2}}$ & $1.37 \times 10^{-1}_{\pm 1.54 \times 10^{-2}}$ & $1.33 \times 10^{-1}_{\pm 1.92 \times 10^{-2}}$ \\
        FiLM+AE         & $7.24 \times 10^{-2}_{\pm 1.02 \times 10^{-1}}$ & $8.06 \times 10^{-2}_{\pm 1.28 \times 10^{-1}}$ & $1.01 \times 10^{-1}_{\pm 1.88 \times 10^{-1}}$ & $1.31 \times 10^{-1}_{\pm 2.73 \times 10^{-1}}$ \\
        SoftNcAE        & $8.95 \times 10^{-2}_{\pm 5.34 \times 10^{-2}}$ & $9.71 \times 10^{-2}_{\pm 7.29 \times 10^{-2}}$ & $9.96 \times 10^{-2}_{\pm 7.29 \times 10^{-2}}$ & $1.02 \times 10^{-1}_{\pm 6.50 \times 10^{-2}}$ \\
        \midrule
        \ncae{} (ours)  & $1.20 \times 10^{-5}_{\pm 2.96 \times 10^{-6}}$ & $1.18 \times 10^{-5}_{\pm 2.96 \times 10^{-6}}$ & $1.18 \times 10^{-5}_{\pm 2.96 \times 10^{-6}}$ & $1.18 \times 10^{-5}_{\pm 2.96 \times 10^{-6}}$ \\
        \bottomrule
    \end{tabular}}
\end{table}
 
% NeurIPS checklist — must be included, does not count toward page limit
% \clearpage
% \input{checklist.tex}

\end{document}